\PassOptionsToPackage{table}{xcolor}
\documentclass{article} 
\usepackage{iclr2027_conference,times}

\usepackage{amsmath,amsfonts,bm}

\def\eqref#1{equation~\ref{#1}}

\def\1{\bm{1}}

\DeclareMathAlphabet{\mathsfit}{\encodingdefault}{\sfdefault}{m}{sl}
\SetMathAlphabet{\mathsfit}{bold}{\encodingdefault}{\sfdefault}{bx}{n}

\usepackage{hyperref}
\usepackage{url}

\hypersetup{
    colorlinks=true,
    linkcolor=blue,       
    citecolor=blue,       
    urlcolor=blue          
}

\usepackage{tikz}
\usetikzlibrary{shapes}
\usepackage[T1]{fontenc}
\usepackage{amsmath}
\usepackage{pgfplots}
\usepackage{booktabs,array} \usepackage[table]{xcolor}
\usepackage{caption}

\usepackage{xcolor}
\definecolor{accent}{RGB}{0,114,178} 
\definecolor{accentwarm}{RGB}{255,140,0} 

\pgfplotsset{compat=1.18}
\usepgfplotslibrary{groupplots}
\usepgfplotslibrary{fillbetween}

\usetikzlibrary{patterns}
\usepackage[T1]{fontenc}
\usepackage{amsmath}
\usepackage{pgfplots}
\pgfplotsset{compat=1.18}
\usepgfplotslibrary{groupplots}
\usepgfplotslibrary{fillbetween}
\usetikzlibrary{patterns}
\definecolor{cinfo}{RGB}{0,114,178}
\definecolor{ccluster}{RGB}{213,94,0}
\definecolor{call}{RGB}{60,60,60}
\definecolor{caccent}{RGB}{0,158,115}
\definecolor{cmuted}{RGB}{140,140,140}
\definecolor{chighlight}{RGB}{204,121,167}
\definecolor{cratio}{RGB}{213,94,0}
\definecolor{cguide}{RGB}{120,120,120}
\definecolor{ccobit}{RGB}{0,114,178}
\definecolor{cclust}{RGB}{213,94,0}
\definecolor{cclust32}{RGB}{0,158,115}
\definecolor{cdata}{RGB}{90,90,90}

\pgfplotsset{
  paperaxis/.style={
    scale only axis,
    axis line style={line width=0.5pt, color=black!70},
    tick style={line width=0.5pt, color=black!70},
    tick align=outside,
    tick pos=left,
    grid=major,
    grid style={line width=0.35pt, color=black!12},
    label style={font=\small},
    tick label style={font=\footnotesize},
    legend style={
      font=\footnotesize,
      draw=black!25,
      line width=0.4pt,
      fill=white,
      fill opacity=0.9,
      text opacity=1,
      inner sep=2.5pt,
      row sep=0.5pt,
    },
    legend cell align=left,
    every axis title/.style={font=\small, yshift=1pt},
  },
}

\tikzset{
  panellabel/.style={font=\small\bfseries, anchor=south west, align=left,
                     xshift=-1.2em, yshift=1.5pt},
}
 \usepackage[T1]{fontenc}
\usepackage{amsmath}
\usepackage{booktabs}
\usepackage{pgfplots}
\pgfplotsset{compat=1.18}
\usepgfplotslibrary{groupplots}
\usepgfplotslibrary{fillbetween}
\usetikzlibrary{patterns}
\definecolor{cinfo}{RGB}{0,114,178}
\definecolor{ccluster}{RGB}{213,94,0}
\definecolor{call}{RGB}{60,60,60}
\definecolor{caccent}{RGB}{0,158,115}
\definecolor{cmuted}{RGB}{140,140,140}
\definecolor{chighlight}{RGB}{204,121,167}
\definecolor{cratio}{RGB}{213,94,0}
\definecolor{cguide}{RGB}{120,120,120}
\definecolor{ccobit}{RGB}{0,114,178}
\definecolor{cclust}{RGB}{213,94,0}
\definecolor{cclust32}{RGB}{0,158,115}
\definecolor{cdata}{RGB}{90,90,90}
\definecolor{crand}{RGB}{204,121,167}
\definecolor{cembed}{RGB}{204,121,167}
\definecolor{cseq1}{RGB}{0,114,178}
\definecolor{cseq2}{RGB}{213,94,0}
\definecolor{cseq3}{RGB}{0,158,115}
\definecolor{cseq4}{RGB}{204,121,167}
\definecolor{cseq5}{RGB}{230,159,0}
\definecolor{cseq6}{RGB}{86,180,233}

\pgfplotsset{
  paperaxis/.style={
    scale only axis,
    axis line style={line width=0.5pt, color=black!70},
    tick style={line width=0.5pt, color=black!70},
    tick align=outside,
    tick pos=left,
    grid=major,
    grid style={line width=0.35pt, color=black!12},
    label style={font=\small},
    tick label style={font=\footnotesize},
    legend style={
      font=\footnotesize,
      draw=black!25,
      line width=0.4pt,
      fill=white,
      fill opacity=0.9,
      text opacity=1,
      inner sep=2.5pt,
      row sep=0.5pt,
    },
    legend cell align=left,
    every axis title/.style={font=\small, yshift=1pt},
  },
}

\tikzset{
  panellabel/.style={font=\small\bfseries, anchor=south west, align=left,
                     xshift=-1.2em, yshift=1.5pt},
}

\usepackage[T1]{fontenc}
\usepackage{amsmath}
\usepackage{booktabs}
\usepackage{pgfplots}
\pgfplotsset{compat=1.18}
\usepgfplotslibrary{groupplots}
\usepgfplotslibrary{fillbetween}
\usetikzlibrary{patterns}
\definecolor{cinfo}{RGB}{0,114,178}
\definecolor{ccluster}{RGB}{213,94,0}
\definecolor{call}{RGB}{60,60,60}
\definecolor{caccent}{RGB}{0,158,115}
\definecolor{cmuted}{RGB}{140,140,140}
\definecolor{chighlight}{RGB}{204,121,167}
\definecolor{cratio}{RGB}{213,94,0}
\definecolor{cguide}{RGB}{120,120,120}
\definecolor{ccobit}{RGB}{0,114,178}
\definecolor{cclust}{RGB}{213,94,0}
\definecolor{cclust32}{RGB}{0,158,115}
\definecolor{cclust_b}{RGB}{204,121,167}
\definecolor{cdata}{RGB}{90,90,90}
\definecolor{crand}{RGB}{204,121,167}
\definecolor{cembed}{RGB}{204,121,167}
\definecolor{cseq1}{RGB}{0,114,178}
\definecolor{cseq2}{RGB}{213,94,0}
\definecolor{cseq3}{RGB}{0,158,115}
\definecolor{cseq4}{RGB}{204,121,167}
\definecolor{cseq5}{RGB}{230,159,0}
\definecolor{cseq6}{RGB}{86,180,233}
\definecolor{cseq7}{RGB}{140,60,110}
\definecolor{cseq8}{RGB}{0,70,110}
\definecolor{cflmbase}{RGB}{30,30,30}

\pgfplotsset{
  paperaxis/.style={
    scale only axis,
    axis line style={line width=0.5pt, color=black!70},
    tick style={line width=0.5pt, color=black!70},
    tick align=outside,
    tick pos=left,
    grid=major,
    grid style={line width=0.35pt, color=black!12},
    label style={font=\small},
    tick label style={font=\footnotesize},
    legend style={
      font=\footnotesize,
      draw=black!25,
      line width=0.4pt,
      fill=white,
      fill opacity=0.9,
      text opacity=1,
      inner sep=2.5pt,
      row sep=0.5pt,
    },
    legend cell align=left,
    every axis title/.style={font=\small, yshift=1pt},
  },
}

\tikzset{
  panellabel/.style={font=\small\bfseries, anchor=south west, align=left,
                     xshift=-1.2em, yshift=1.5pt},
}

\usepackage[T1]{fontenc}
\usepackage{amsmath}
\usepackage{booktabs}
\usepackage{pgfplots}
\pgfplotsset{compat=1.18}
\usepgfplotslibrary{groupplots}
\usepgfplotslibrary{fillbetween}
\usetikzlibrary{patterns}
\usetikzlibrary{arrows.meta}
\definecolor{cinfo}{RGB}{0,114,178}
\definecolor{ccluster}{RGB}{213,94,0}
\definecolor{call}{RGB}{60,60,60}
\definecolor{caccent}{RGB}{0,158,115}
\definecolor{cmuted}{RGB}{140,140,140}
\definecolor{chighlight}{RGB}{204,121,167}
\definecolor{cratio}{RGB}{213,94,0}
\definecolor{cguide}{RGB}{120,120,120}
\definecolor{ccobit}{RGB}{0,114,178}
\definecolor{cclust}{RGB}{213,94,0}
\definecolor{cclust32}{RGB}{0,158,115}
\definecolor{cclust_b}{RGB}{204,121,167}
\definecolor{cdata}{RGB}{90,90,90}
\definecolor{crand}{RGB}{204,121,167}
\definecolor{cembed}{RGB}{204,121,167}
\definecolor{cseq1}{RGB}{0,114,178}
\definecolor{cseq2}{RGB}{213,94,0}
\definecolor{cseq3}{RGB}{0,158,115}
\definecolor{cseq4}{RGB}{204,121,167}
\definecolor{cseq5}{RGB}{230,159,0}
\definecolor{cseq6}{RGB}{86,180,233}
\definecolor{cseq7}{RGB}{140,60,110}
\definecolor{cseq8}{RGB}{0,70,110}
\definecolor{cflmbase}{RGB}{30,30,30}
\definecolor{cchurn1}{RGB}{230,159,0}
\definecolor{cchurn2}{RGB}{86,180,233}
\definecolor{ccobitsn}{RGB}{86,180,233}
\definecolor{cchurnwin}{RGB}{0,158,115}
\definecolor{cchurnfz}{RGB}{230,159,0}
\definecolor{cchurncln}{RGB}{0,158,115}

\pgfplotsset{
  paperaxis/.style={
    scale only axis,
    axis line style={line width=0.5pt, color=black!70},
    tick style={line width=0.5pt, color=black!70},
    tick align=outside,
    tick pos=left,
    grid=major,
    grid style={line width=0.35pt, color=black!12},
    label style={font=\small},
    tick label style={font=\footnotesize},
    legend style={
      font=\footnotesize,
      draw=black!25,
      line width=0.4pt,
      fill=white,
      fill opacity=0.9,
      text opacity=1,
      inner sep=2.5pt,
      row sep=0.5pt,
    },
    legend cell align=left,
    every axis title/.style={font=\small, yshift=1pt},
  },
}

\tikzset{
  panellabel/.style={font=\small\bfseries, anchor=south west, align=left,
                     xshift=-1.2em, yshift=1.5pt},
}
 
\usepackage[T1]{fontenc}
\usepackage{amsmath}
\usepackage{booktabs}
\usepackage{pgfplots}
\pgfplotsset{compat=1.18}
\usepgfplotslibrary{groupplots}
\usepgfplotslibrary{fillbetween}
\usetikzlibrary{patterns}
\definecolor{cinfo}{RGB}{0,114,178}
\definecolor{ccluster}{RGB}{213,94,0}
\definecolor{call}{RGB}{60,60,60}
\definecolor{caccent}{RGB}{0,158,115}
\definecolor{cmuted}{RGB}{140,140,140}
\definecolor{chighlight}{RGB}{204,121,167}
\definecolor{cratio}{RGB}{213,94,0}
\definecolor{cguide}{RGB}{120,120,120}
\definecolor{ccobit}{RGB}{0,114,178}
\definecolor{cclust}{RGB}{213,94,0}
\definecolor{cclust32}{RGB}{0,158,115}
\definecolor{cdata}{RGB}{90,90,90}
\definecolor{crand}{RGB}{204,121,167}
\definecolor{cembed}{RGB}{204,121,167}
\definecolor{cseq1}{RGB}{0,114,178}
\definecolor{cseq2}{RGB}{213,94,0}
\definecolor{cseq3}{RGB}{0,158,115}
\definecolor{cseq4}{RGB}{204,121,167}
\definecolor{cseq5}{RGB}{230,159,0}
\definecolor{cseq6}{RGB}{86,180,233}

\pgfplotsset{
  paperaxis/.style={
    scale only axis,
    axis line style={line width=0.5pt, color=black!70},
    tick style={line width=0.5pt, color=black!70},
    tick align=outside,
    tick pos=left,
    grid=major,
    grid style={line width=0.35pt, color=black!12},
    label style={font=\small},
    tick label style={font=\footnotesize},
    legend style={
      font=\footnotesize,
      draw=black!25,
      line width=0.4pt,
      fill=white,
      fill opacity=0.9,
      text opacity=1,
      inner sep=2.5pt,
      row sep=0.5pt,
    },
    legend cell align=left,
    every axis title/.style={font=\small, yshift=1pt},
  },
}

\tikzset{
  panellabel/.style={font=\small\bfseries, anchor=south west, align=left,
                     xshift=-1.2em, yshift=1.5pt},
}
 
\usepackage[T1]{fontenc}
\usepackage{amsmath}
\usepackage{booktabs}
\usepackage{pgfplots}
\pgfplotsset{compat=1.18}
\usepgfplotslibrary{groupplots}
\usepgfplotslibrary{fillbetween}
\usetikzlibrary{patterns}
\usetikzlibrary{arrows.meta}
\definecolor{cinfo}{RGB}{0,114,178}
\definecolor{ccluster}{RGB}{213,94,0}
\definecolor{call}{RGB}{60,60,60}
\definecolor{caccent}{RGB}{0,158,115}
\definecolor{cmuted}{RGB}{140,140,140}
\definecolor{chighlight}{RGB}{204,121,167}
\definecolor{cratio}{RGB}{213,94,0}
\definecolor{cguide}{RGB}{120,120,120}
\definecolor{ccobit}{RGB}{0,114,178}
\definecolor{cclust}{RGB}{213,94,0}
\definecolor{cclust32}{RGB}{0,158,115}
\definecolor{cclust_b}{RGB}{204,121,167}
\definecolor{cdata}{RGB}{90,90,90}
\definecolor{crand}{RGB}{204,121,167}
\definecolor{cembed}{RGB}{204,121,167}
\definecolor{cprefix}{RGB}{0,158,115}
\definecolor{cseq1}{RGB}{0,114,178}
\definecolor{cseq2}{RGB}{213,94,0}
\definecolor{cseq3}{RGB}{0,158,115}
\definecolor{cseq4}{RGB}{204,121,167}
\definecolor{cseq5}{RGB}{230,159,0}
\definecolor{cseq6}{RGB}{86,180,233}
\definecolor{cseq7}{RGB}{140,60,110}
\definecolor{cseq8}{RGB}{0,70,110}
\definecolor{cflmbase}{RGB}{30,30,30}
\definecolor{cchurn1}{RGB}{230,159,0}
\definecolor{cchurn2}{RGB}{86,180,233}
\definecolor{ccobitsn}{RGB}{86,180,233}
\definecolor{cchurnwin}{RGB}{0,158,115}
\definecolor{cchurnfz}{RGB}{230,159,0}
\definecolor{cchurncln}{RGB}{0,158,115}

\pgfplotsset{
  paperaxis/.style={
    scale only axis,
    axis line style={line width=0.5pt, color=black!70},
    tick style={line width=0.5pt, color=black!70},
    tick align=outside,
    tick pos=left,
    grid=major,
    grid style={line width=0.35pt, color=black!12},
    label style={font=\small},
    tick label style={font=\footnotesize},
    legend style={
      font=\footnotesize,
      draw=black!25,
      line width=0.4pt,
      fill=white,
      fill opacity=0.9,
      text opacity=1,
      inner sep=2.5pt,
      row sep=0.5pt,
    },
    legend cell align=left,
    every axis title/.style={font=\small, yshift=1pt},
  },
}

\tikzset{
  panellabel/.style={font=\small\bfseries, anchor=south west, align=left,
                     xshift=-1.2em, yshift=1.5pt},
}
 
\usepackage[T1]{fontenc}
\usepackage{amsmath}
\usepackage{booktabs}
\usepackage{pgfplots}
\pgfplotsset{compat=1.18}
\usepgfplotslibrary{groupplots}
\usepgfplotslibrary{fillbetween}
\usetikzlibrary{patterns}
\definecolor{cinfo}{RGB}{0,114,178}
\definecolor{ccluster}{RGB}{213,94,0}
\definecolor{call}{RGB}{60,60,60}
\definecolor{caccent}{RGB}{0,158,115}
\definecolor{cmuted}{RGB}{140,140,140}
\definecolor{chighlight}{RGB}{204,121,167}
\definecolor{cratio}{RGB}{213,94,0}
\definecolor{cguide}{RGB}{120,120,120}
\definecolor{ccobit}{RGB}{0,114,178}
\definecolor{cclust}{RGB}{213,94,0}
\definecolor{cclust32}{RGB}{0,158,115}
\definecolor{cdata}{RGB}{90,90,90}
\definecolor{crand}{RGB}{204,121,167}
\definecolor{cembed}{RGB}{204,121,167}
\definecolor{cseq1}{RGB}{0,114,178}
\definecolor{cseq2}{RGB}{213,94,0}
\definecolor{cseq3}{RGB}{0,158,115}
\definecolor{cseq4}{RGB}{204,121,167}
\definecolor{cseq5}{RGB}{230,159,0}
\definecolor{cseq6}{RGB}{86,180,233}

\pgfplotsset{
  paperaxis/.style={
    scale only axis,
    axis line style={line width=0.5pt, color=black!70},
    tick style={line width=0.5pt, color=black!70},
    tick align=outside,
    tick pos=left,
    grid=major,
    grid style={line width=0.35pt, color=black!12},
    label style={font=\small},
    tick label style={font=\footnotesize},
    legend style={
      font=\footnotesize,
      draw=black!25,
      line width=0.4pt,
      fill=white,
      fill opacity=0.9,
      text opacity=1,
      inner sep=2.5pt,
      row sep=0.5pt,
    },
    legend cell align=left,
    every axis title/.style={font=\small, yshift=1pt},
  },
}

\tikzset{
  panellabel/.style={font=\small\bfseries, anchor=south west, align=left,
                     xshift=-1.2em, yshift=1.5pt},
}

\usepackage[T1]{fontenc}
\usepackage{amsmath}
\usepackage{booktabs}
\usepackage{pgfplots}
\pgfplotsset{compat=1.18}
\usepgfplotslibrary{groupplots}
\usepgfplotslibrary{fillbetween}
\usetikzlibrary{patterns}
\definecolor{cinfo}{RGB}{0,114,178}
\definecolor{ccluster}{RGB}{213,94,0}
\definecolor{call}{RGB}{60,60,60}
\definecolor{caccent}{RGB}{0,158,115}
\definecolor{cmuted}{RGB}{140,140,140}
\definecolor{chighlight}{RGB}{204,121,167}
\definecolor{cratio}{RGB}{213,94,0}
\definecolor{cguide}{RGB}{120,120,120}
\definecolor{ccobit}{RGB}{0,114,178}
\definecolor{cclust}{RGB}{213,94,0}
\definecolor{cclust32}{RGB}{0,158,115}
\definecolor{cclust_b}{RGB}{204,121,167}
\definecolor{cdata}{RGB}{90,90,90}
\definecolor{crand}{RGB}{204,121,167}
\definecolor{cembed}{RGB}{204,121,167}
\definecolor{cseq1}{RGB}{0,114,178}
\definecolor{cseq2}{RGB}{213,94,0}
\definecolor{cseq3}{RGB}{0,158,115}
\definecolor{cseq4}{RGB}{204,121,167}
\definecolor{cseq5}{RGB}{230,159,0}
\definecolor{cseq6}{RGB}{86,180,233}
\definecolor{cseq7}{RGB}{140,60,110}
\definecolor{cseq8}{RGB}{0,70,110}
\definecolor{cflmbase}{RGB}{30,30,30}
\definecolor{cchurn1}{RGB}{230,159,0}
\definecolor{cchurn2}{RGB}{86,180,233}
\definecolor{ccobitsn}{RGB}{86,180,233}
\definecolor{cchurnwin}{RGB}{0,158,115}
\definecolor{cchurnfz}{RGB}{230,159,0}
\definecolor{cchurncln}{RGB}{0,158,115}

\pgfplotsset{
  paperaxis/.style={
    scale only axis,
    axis line style={line width=0.5pt, color=black!70},
    tick style={line width=0.5pt, color=black!70},
    tick align=outside,
    tick pos=left,
    grid=major,
    grid style={line width=0.35pt, color=black!12},
    label style={font=\small},
    tick label style={font=\footnotesize},
    legend style={
      font=\footnotesize,
      draw=black!25,
      line width=0.4pt,
      fill=white,
      fill opacity=0.9,
      text opacity=1,
      inner sep=2.5pt,
      row sep=0.5pt,
    },
    legend cell align=left,
    every axis title/.style={font=\small, yshift=1pt},
  },
}

\tikzset{
  panellabel/.style={font=\small\bfseries, anchor=south west, align=left,
                     xshift=-1.2em, yshift=1.5pt},
}
 
\usepackage[T1]{fontenc}
\usepackage{amsmath}
\usepackage{booktabs}
\usepackage{pgfplots}
\pgfplotsset{compat=1.18}
\usepgfplotslibrary{groupplots}
\usepgfplotslibrary{fillbetween}
\usetikzlibrary{patterns}
\definecolor{cinfo}{RGB}{0,114,178}
\definecolor{ccluster}{RGB}{213,94,0}
\definecolor{call}{RGB}{60,60,60}
\definecolor{caccent}{RGB}{0,158,115}
\definecolor{cmuted}{RGB}{140,140,140}
\definecolor{chighlight}{RGB}{204,121,167}
\definecolor{cratio}{RGB}{213,94,0}
\definecolor{cguide}{RGB}{120,120,120}
\definecolor{ccobit}{RGB}{0,114,178}
\definecolor{cclust}{RGB}{213,94,0}
\definecolor{cclust32}{RGB}{0,158,115}
\definecolor{cclust_b}{RGB}{204,121,167}
\definecolor{cdata}{RGB}{90,90,90}
\definecolor{crand}{RGB}{204,121,167}
\definecolor{cembed}{RGB}{204,121,167}
\definecolor{cseq1}{RGB}{0,114,178}
\definecolor{cseq2}{RGB}{213,94,0}
\definecolor{cseq3}{RGB}{0,158,115}
\definecolor{cseq4}{RGB}{204,121,167}
\definecolor{cseq5}{RGB}{230,159,0}
\definecolor{cseq6}{RGB}{86,180,233}
\definecolor{cseq7}{RGB}{140,60,110}
\definecolor{cseq8}{RGB}{0,70,110}
\definecolor{cflmbase}{RGB}{30,30,30}

\pgfplotsset{
  paperaxis/.style={
    scale only axis,
    axis line style={line width=0.5pt, color=black!70},
    tick style={line width=0.5pt, color=black!70},
    tick align=outside,
    tick pos=left,
    grid=major,
    grid style={line width=0.35pt, color=black!12},
    label style={font=\small},
    tick label style={font=\footnotesize},
    legend style={
      font=\footnotesize,
      draw=black!25,
      line width=0.4pt,
      fill=white,
      fill opacity=0.9,
      text opacity=1,
      inner sep=2.5pt,
      row sep=0.5pt,
    },
    legend cell align=left,
    every axis title/.style={font=\small, yshift=1pt},
  },
}

\tikzset{
  panellabel/.style={font=\small\bfseries, anchor=south west, align=left,
                     xshift=-1.2em, yshift=1.5pt},
}

\usepackage[T1]{fontenc}
\usepackage{amsmath}
\usepackage{amssymb}
\usepackage{bm}
\usepackage{tikz}
\usetikzlibrary{arrows.meta}
\usetikzlibrary{decorations.pathreplacing}
\usetikzlibrary{positioning}
\usetikzlibrary{shapes.geometric}
\definecolor{cinfo}{RGB}{0,114,178}
\definecolor{ccluster}{RGB}{213,94,0}
\definecolor{caccent}{RGB}{0,158,115}

\tikzset{
  hcd/ptitle/.style ={font=\scriptsize\bfseries, anchor=north west, text=black!80},
  hcd/stitle/.style ={font=\tiny\bfseries, anchor=north west, text=black!80},
  hcd/note/.style   ={font=\tiny, text=black!60},
  hcd/word/.style   ={draw=cinfo!45, line width=0.4pt, fill=cinfo!12,
                      rounded corners=1pt, minimum width=0.72cm,
                      minimum height=0.30cm, inner sep=0.5pt, font=\tiny},
  hcd/clus/.style   ={draw=ccluster!50, line width=0.4pt, fill=ccluster!15,
                      rounded corners=1pt, minimum width=0.72cm,
                      minimum height=0.30cm, inner sep=0.5pt, font=\tiny},
  hcd/ar/.style     ={-{Stealth[length=3.6pt]}, draw=black!55, line width=0.5pt,
                      shorten >=1pt, shorten <=1pt},
  hcd/thin/.style   ={-{Stealth[length=2.8pt]}, draw=black!40, line width=0.35pt,
                      shorten >=1.4pt, shorten <=1.4pt},
  hcd/net/.style    ={draw=black!55, line width=0.6pt, fill=black!4,
                      rounded corners=2pt, align=center, font=\scriptsize},
  hcd/box/.style    ={draw=black!15, rounded corners=2pt, fill=black!2,
                      inner sep=3pt, font=\tiny},
  hcd/panel/.style  ={draw=black!12, rounded corners=2.5pt, fill=black!2,
                      inner sep=0pt},
  hcd/frame/.style  ={draw=black!18, rounded corners=2pt, inner sep=0pt},
  hcd/sep/.style    ={draw=black!15, line width=0.4pt},
  hcd/dot/.style    ={circle, inner sep=0pt, minimum size=2.2pt},
  hcd/ring/.style   ={draw=black!35, densely dashed, line width=0.35pt, ellipse},
}

\newcommand{\HcdCol}[4]{%
  \fill[#4, draw=ccluster!70, line width=0.3pt] (#1,#2) rectangle ++(0.22,0.22);
  \draw[white, line width=0.35pt] (#1,#2+0.11) -- ++(0.22,0);
  \fill[#3, draw=cinfo!60, line width=0.3pt] (#1,#2+0.26) rectangle ++(0.22,0.44);
  \foreach \k in {1,2,3}
    \draw[white, line width=0.35pt] (#1,#2+0.26+0.11*\k) -- ++(0.22,0);
}
\newcommand{\HcdTokCell}[2]{%
  \fill[cinfo!22, draw=cinfo!60, line width=0.3pt] (#1,#2) rectangle ++(0.22,0.44);
  \foreach \k in {1,2,3}
    \draw[white, line width=0.35pt] (#1,#2+0.11*\k) -- ++(0.22,0);
}
\newcommand{\HcdClusCell}[2]{%
  \fill[ccluster!30, draw=ccluster!70, line width=0.3pt] (#1,#2) rectangle ++(0.22,0.22);
  \draw[white, line width=0.35pt] (#1,#2+0.11) -- ++(0.22,0);
}

\usepackage[T1]{fontenc}
\usepackage{amsmath}
\usepackage{booktabs}
\usepackage{pgfplots}
\pgfplotsset{compat=1.18}
\usepgfplotslibrary{groupplots}
\usepgfplotslibrary{fillbetween}
\usetikzlibrary{patterns}
\usetikzlibrary{arrows.meta}
\definecolor{cinfo}{RGB}{0,114,178}
\definecolor{ccluster}{RGB}{213,94,0}
\definecolor{call}{RGB}{60,60,60}
\definecolor{caccent}{RGB}{0,158,115}
\definecolor{cmuted}{RGB}{140,140,140}
\definecolor{chighlight}{RGB}{204,121,167}
\definecolor{cratio}{RGB}{213,94,0}
\definecolor{cguide}{RGB}{120,120,120}
\definecolor{ccobit}{RGB}{0,114,178}
\definecolor{cclust}{RGB}{213,94,0}
\definecolor{cclust32}{RGB}{0,158,115}
\definecolor{cclust_b}{RGB}{204,121,167}
\definecolor{cdata}{RGB}{90,90,90}
\definecolor{crand}{RGB}{204,121,167}
\definecolor{cembed}{RGB}{204,121,167}
\definecolor{cprefix}{RGB}{0,158,115}
\definecolor{cseq1}{RGB}{0,114,178}
\definecolor{cseq2}{RGB}{213,94,0}
\definecolor{cseq3}{RGB}{0,158,115}
\definecolor{cseq4}{RGB}{204,121,167}
\definecolor{cseq5}{RGB}{230,159,0}
\definecolor{cseq6}{RGB}{86,180,233}
\definecolor{cseq7}{RGB}{140,60,110}
\definecolor{cseq8}{RGB}{0,70,110}
\definecolor{cflmbase}{RGB}{30,30,30}
\definecolor{cchurn1}{RGB}{230,159,0}
\definecolor{cchurn2}{RGB}{86,180,233}
\definecolor{ccobitsn}{RGB}{86,180,233}
\definecolor{cchurnwin}{RGB}{0,158,115}
\definecolor{cchurnfz}{RGB}{230,159,0}
\definecolor{cchurncln}{RGB}{0,158,115}
\definecolor{cflmdet}{RGB}{0,114,178}
\definecolor{chflmsto}{RGB}{213,94,0}

\pgfplotsset{
  paperaxis/.style={
    scale only axis,
    axis line style={line width=0.5pt, color=black!70},
    tick style={line width=0.5pt, color=black!70},
    tick align=outside,
    tick pos=left,
    grid=major,
    grid style={line width=0.35pt, color=black!12},
    label style={font=\small},
    tick label style={font=\footnotesize},
    legend style={
      font=\footnotesize,
      draw=black!25,
      line width=0.4pt,
      fill=white,
      fill opacity=0.9,
      text opacity=1,
      inner sep=2.5pt,
      row sep=0.5pt,
    },
    legend cell align=left,
    every axis title/.style={font=\small, yshift=1pt},
  },
}

\tikzset{
  panellabel/.style={font=\small\bfseries, anchor=south west, align=left,
                     xshift=-1.2em, yshift=1.5pt},
}

\usepackage[T1]{fontenc}
\usepackage{amsmath}
\usepackage{booktabs}
\usepackage{pgfplots}
\pgfplotsset{compat=1.18}
\usepgfplotslibrary{groupplots}
\usepgfplotslibrary{fillbetween}
\usetikzlibrary{patterns}
\usetikzlibrary{arrows.meta}
\definecolor{cinfo}{RGB}{0,114,178}
\definecolor{ccluster}{RGB}{213,94,0}
\definecolor{call}{RGB}{60,60,60}
\definecolor{caccent}{RGB}{0,158,115}
\definecolor{cmuted}{RGB}{140,140,140}
\definecolor{chighlight}{RGB}{204,121,167}
\definecolor{cratio}{RGB}{213,94,0}
\definecolor{cguide}{RGB}{120,120,120}
\definecolor{ccobit}{RGB}{0,114,178}
\definecolor{cclust}{RGB}{213,94,0}
\definecolor{cclust32}{RGB}{0,158,115}
\definecolor{cclust_b}{RGB}{204,121,167}
\definecolor{cdata}{RGB}{90,90,90}
\definecolor{crand}{RGB}{204,121,167}
\definecolor{cembed}{RGB}{204,121,167}
\definecolor{cseq1}{RGB}{0,114,178}
\definecolor{cseq2}{RGB}{213,94,0}
\definecolor{cseq3}{RGB}{0,158,115}
\definecolor{cseq4}{RGB}{204,121,167}
\definecolor{cseq5}{RGB}{230,159,0}
\definecolor{cseq6}{RGB}{86,180,233}
\definecolor{cseq7}{RGB}{140,60,110}
\definecolor{cseq8}{RGB}{0,70,110}
\definecolor{cflmbase}{RGB}{30,30,30}
\definecolor{cchurn1}{RGB}{230,159,0}
\definecolor{cchurn2}{RGB}{86,180,233}
\definecolor{ccobitsn}{RGB}{86,180,233}
\definecolor{cchurnwin}{RGB}{0,158,115}
\definecolor{cchurnfz}{RGB}{230,159,0}
\definecolor{cchurncln}{RGB}{0,158,115}

\pgfplotsset{
  paperaxis/.style={
    scale only axis,
    axis line style={line width=0.5pt, color=black!70},
    tick style={line width=0.5pt, color=black!70},
    tick align=outside,
    tick pos=left,
    grid=major,
    grid style={line width=0.35pt, color=black!12},
    label style={font=\small},
    tick label style={font=\footnotesize},
    legend style={
      font=\footnotesize,
      draw=black!25,
      line width=0.4pt,
      fill=white,
      fill opacity=0.9,
      text opacity=1,
      inner sep=2.5pt,
      row sep=0.5pt,
    },
    legend cell align=left,
    every axis title/.style={font=\small, yshift=1pt},
  },
}

\tikzset{
  panellabel/.style={font=\small\bfseries, anchor=south west, align=left,
                     xshift=-1.2em, yshift=1.5pt},
}

\title{Denoising Hierarchical Representations: Joint Continuous Diffusion for Language Modeling}

\author{Mathias Ollu \\
Ecole Polytechnique\thanks{Work done at Archimedes, Athena RC, Greece}\\
\texttt{mathias.ollu}@polytechnique.edu
\And
Nikos Komodakis \\
Archimedes, Athena RC\\
University of Crete\\
IACM-Forth\\
\texttt{komod}@csd.uoc.gr
}

\definecolor{cinfo}{RGB}{0,114,178}
\definecolor{ccluster}{RGB}{213,94,0}
\definecolor{call}{RGB}{60,60,60}
\definecolor{caccent}{RGB}{0,158,115}
\definecolor{cmuted}{RGB}{140,140,140}
\definecolor{chighlight}{RGB}{204,121,167}
\definecolor{cratio}{RGB}{213,94,0}
\definecolor{cguide}{RGB}{120,120,120}

\usepackage[utf8]{inputenc}
\usepackage[T1]{fontenc}
\usepackage{microtype}
\usepackage{amsmath, amssymb, amsfonts, mathtools}
\usepackage{booktabs}
\usepackage{array}
\usepackage{tabularx}
\usepackage{longtable}
\usepackage{multirow}
\usepackage{graphicx}
\usepackage{subcaption}
\usepackage{xcolor}
\usepackage{ragged2e}
\usepackage{enumitem}
\usepackage{algorithm}
\usepackage{algpseudocode}
\usepackage{siunitx}
\usepackage{wrapfig}
\usepackage{url}
\usepackage{hyperref}
\usepackage{placeins}
\usepackage[capitalize,noabbrev]{cleveref}
\usepackage{tikz}
\usetikzlibrary{
    positioning, 
    arrows.meta, 
    shapes.geometric, 
    fit,
    calc,                     
    decorations.pathreplacing,
    backgrounds               
}

\usepackage{amsmath, amssymb, amsfonts, amsthm, mathtools}
\newtheorem{proposition}{Proposition}

\newcommand{\clip}{\operatorname{clip}}

\newcommand{\para}[1]{\vspace{0.25em}\noindent\textbf{#1}}
\tikzset{
  panellabel/.style={font=\small\bfseries, anchor=south west,
                     xshift=-1.2em, yshift=1.5pt},
}
\usepackage{tikz}

\pgfplotsset{
  paperaxis/.style={
    scale only axis,
    axis line style={line width=0.5pt, color=black!70},
    tick style={line width=0.5pt, color=black!70},
    tick align=outside,
    tick pos=left,
    grid=major,
    grid style={line width=0.35pt, color=black!12},
    label style={font=\small},
    tick label style={font=\footnotesize},
    legend style={
      font=\footnotesize,
      draw=black!25,
      line width=0.4pt,
      fill=white,
      fill opacity=0.9,
      text opacity=1,
      inner sep=2.5pt,
      row sep=0.5pt,
    },
    legend cell align=left,
    every axis title/.style={font=\small, yshift=1pt},
  },
}

\iclrfinalcopy 
\begin{document}

\maketitle

\begin{abstract}
Diffusion Language Models (DLMs) hold the promise of order-agnostic, parallel text generation. Recently, continuous diffusion and flow matching models have seen substantial gains, driven by carefully crafted token representations and diffusion/flow spaces. In this work, we introduce 
\textbf{Hierarchical Continuous Diffusion Language Models} (H-CDLMs), 
a simple 
framework
that further improves continuous DLMs with minimal compute and parameter overhead. Drawing on the discrete DLM and continuous image diffusion literature on \emph{joint diffusion}, we diffuse multiple modalities in parallel. These modalities represent tokens at different semantic granularities: in our instantiation, the tokens themselves and coarser \emph{clusters} obtained by clustering pretrained token embeddings.
We propose a general setup that allows per-modality samplers and 
schedules
to enhance the interplay between modalities.
Applied to CoBit, this yields H-CoBit, which delivers large empirical gains 
across benchmarks. 
At dataset entropy, H-CoBit improves MAUVE and reaches a generative perplexity (GenPPL) of 49.4 on LM1B and 50.4 on OWT, improving on the baseline by 24.2 and 20.7 points and surpassing even discrete DLMs of comparable size. On GSM8K, it reaches 27.4\% accuracy, outperforming prior continuous diffusion and flow-based models.
We further apply H-CDLM to the flow matching model FLM, obtaining consistent gains with H-FLM and demonstrating that the framework generalizes across continuous generative paradigms. Our code will be made publicly available at \href{https://github.com/matol-16/HCDLM.git}{Github}.

\end{abstract}

\begin{center}
    \nopagebreak 
\begin{tikzpicture}
  \begin{groupplot}[
    group style={group size=2 by 1, horizontal sep=1.45cm, vertical sep=1.15cm},
  ]
    \nextgroupplot[
      paperaxis,
      width=4.908743896499239cm,
      height=3.2445cm,
      xmin=0.86,
      xmax=3.40669,
      ymin=-0.5,
      ymax=2.65,
      xlabel={compute saving \ (better $\rightarrow$)},
      axis lines*=left,
      xtick={1,1.5,2,3},
      xticklabels={$1\times$,$1.5\times$,$2\times$,$3\times$},
      xminorticks=false,
      ytick={0,1,2},
      yticklabels={GSM8K,OpenWebText,LM1B},
      ytick style={draw=none},
      y tick label style={font=\small},
      xmajorgrids=true,
      ymajorgrids=false,
    ]
    \draw[cdata, line width=0.7pt, densely dashed] (axis cs:1,-0.5) -- (axis cs:1,2.65);
    \node[circle, fill=cdata, inner sep=2pt] at (axis cs:1,2) {};
    \node[circle, fill=cseq1, inner sep=2.6pt] at (axis cs:0.971534,2) {};
    \node[cseq1, font=\small, anchor=west] at (axis cs:1.055,2) {$\approx$ same compute};
    \draw[cseq3, line width=1.4pt, -{Stealth[length=5pt,width=4pt]}, shorten <=2.5pt, shorten >=4pt] (axis cs:1,1) -- (axis cs:1.91855,1);
    \node[circle, fill=cdata, inner sep=2pt] at (axis cs:1,1) {};
    \node[circle, fill=cseq3, inner sep=2.6pt] at (axis cs:1.91855,1) {};
    \node[cseq3, font=\small, anchor=south west] at (axis cs:1.055,1.12) {1.9$\times$ less compute};
    \draw[cseq2, line width=1.4pt, -{Stealth[length=5pt,width=4pt]}, shorten <=2.5pt, shorten >=4pt] (axis cs:1,0) -- (axis cs:2.58083,0);
    \node[circle, fill=cdata, inner sep=2pt] at (axis cs:1,0) {};
    \node[circle, fill=cseq2, inner sep=2.6pt] at (axis cs:2.58083,0) {};
    \node[cseq2, font=\small, anchor=south west] at (axis cs:1.055,0.12) {2.6$\times$ less compute};
    \nextgroupplot[
      paperaxis,
      width=4.908743896499239cm,
      height=3.2445cm,
      xmin=-1.5,
      xmax=38.8977,
      ymin=-0.5,
      ymax=2.65,
      xlabel={relative quality gain \ (better $\rightarrow$)},
      axis lines*=left,
      xtick={0,10,20,30},
      xticklabels={0\%,10\%,20\%,30\%},
      xminorticks=false,
      ytick={0,1,2},
      yticklabels={},
      ytick style={draw=none},
      y tick label style={font=\small},
      xmajorgrids=true,
      ymajorgrids=false,
      clip=false,
    ]
    \draw[cdata, line width=0.7pt, densely dashed] (axis cs:0,-0.5) -- (axis cs:0,2.65);
    \draw[cseq1, line width=1.4pt, -{Stealth[length=5pt,width=4pt]}, shorten <=2.5pt, shorten >=4pt] (axis cs:0,2) -- (axis cs:32.9641,2);
    \node[circle, fill=cdata, inner sep=2pt] at (axis cs:0,2) {};
    \node[circle, fill=cseq1, inner sep=2.6pt] at (axis cs:32.9641,2) {};
    \node[cseq1, font=\small\bfseries, anchor=south west, inner sep=1.5pt] at (axis cs:0.727158,2.08) {$-$33\% GenPPL};
    \node[cseq1, font=\scriptsize, anchor=north west, inner sep=1.5pt] at (axis cs:0.727158,1.92) {73.6 $\rightarrow$ 49.4};
    \draw[cseq3, line width=1.4pt, -{Stealth[length=5pt,width=4pt]}, shorten <=2.5pt, shorten >=4pt] (axis cs:0,1) -- (axis cs:29.1472,1);
    \node[circle, fill=cdata, inner sep=2pt] at (axis cs:0,1) {};
    \node[circle, fill=cseq3, inner sep=2.6pt] at (axis cs:29.1472,1) {};
    \node[cseq3, font=\small\bfseries, anchor=south west, inner sep=1.5pt] at (axis cs:0.727158,1.08) {$-$29\% GenPPL};
    \node[cseq3, font=\scriptsize, anchor=north west, inner sep=1.5pt] at (axis cs:0.727158,0.92) {71.1 $\rightarrow$ 50.4};
    \draw[cseq2, line width=1.4pt, -{Stealth[length=5pt,width=4pt]}, shorten <=2.5pt, shorten >=4pt] (axis cs:0,0) -- (axis cs:8.05132,0);
    \node[circle, fill=cdata, inner sep=2pt] at (axis cs:0,0) {};
    \node[circle, fill=cseq2, inner sep=2.6pt] at (axis cs:8.05132,0) {};
    \node[cseq2, font=\small\bfseries, anchor=south west, inner sep=1.5pt] at (axis cs:0.727158,0.08) {$+$8\% accuracy};
    \node[cseq2, font=\scriptsize, anchor=north west, inner sep=1.5pt] at (axis cs:0.727158,-0.08) {25.4\% $\rightarrow$ 27.4\% ($+$2\,pt)};
    \node[circle, fill=cdata, inner sep=2pt, yshift=4.5pt] at (rel axis cs:0.42,1) {};
    \node[anchor=base west, yshift=1.5pt, font=\small, inner sep=1pt] at (rel axis cs:0.447,1) {CoBit};
    \node[circle, fill=black!85, inner sep=2.6pt, yshift=4.5pt] at (rel axis cs:0.7,1) {};
    \node[anchor=west, yshift=4.5pt, align=left, font=\small, inner sep=1pt] at (rel axis cs:0.728,1) {H-CoBit \\ (coloured)};
  \end{groupplot}
  \node[panellabel] at (group c1r1.north west) {Faster};
  \node[panellabel] at (group c2r1.north west) {Better};
\end{tikzpicture}
    
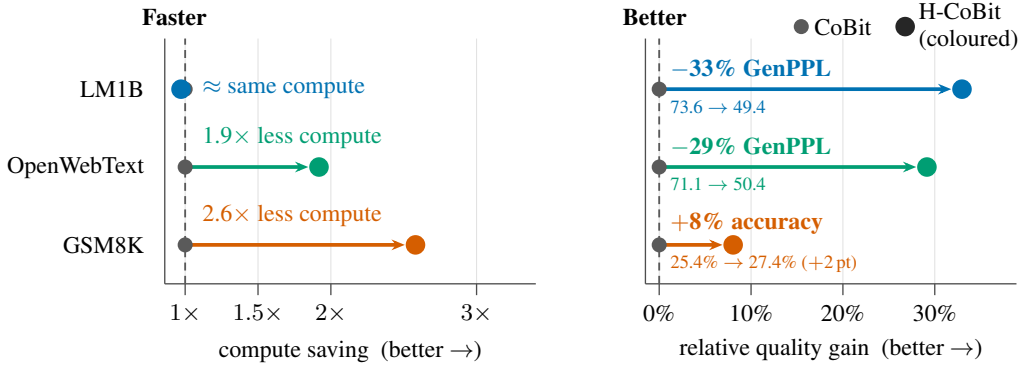
\captionof{figure}{\textbf{Better generations for less compute.} For each dataset, \textbf{(left)} the compute H-CoBit (ours) needs to match/surpass \emph{CoBit's best results}, relative to CoBit's own cost. Compute is each arm's measured per-denoiser-call cost times its NFE; \textbf{(right)} the quality gain for optimal samplers, Gen-PPL is read at dataset entropy, GSM8K reports accuracy. LM1B is not budget-matched (CoBit 1M steps vs H-CoBit 600k).}

  \label{fig:teaser}
  \vspace{-6pt}
\end{center}

\section{Introduction}
\vspace{-4pt}
Recent progress in language modeling \citep{touvron2023llama2openfoundation} was enabled by the simple and scalable auto-regressive (AR) factorization that supports left-to-right generation via next-token prediction. This sequential nature, however, precludes parallelization and can hinder performance on non-left-to-right tasks. Diffusion Language Models (DLMs) offer a promising alternative by denoising entire sequences in parallel \citep{austin2021d3pm, sahoo2024mdlm}. Yet realizing their potential is difficult: diffusion (or flow) architectures lack a single canonical training recipe, and neither discrete nor continuous DLMs have fully closed the gap with AR models despite success in other modalities.
A key design choice is how to represent and corrupt tokens. Discrete approaches such as Masked Diffusion Language Models (MDLM) have narrowed the gap \citep{sahoo2024mdlm, austin2021d3pm, sahoo2025duo}. Continuous formulations are harder: the design space is larger, and standard Gaussian diffusion with linear schedules is poorly matched to discrete data \citep{shabalin2026gaussiandiffusionmodelsfail}. Recent work has nevertheless closed much of this gap through carefully chosen continuous token representations and adaptive noise schedules \citep{chen2026langflow, lee2026flm}.

Our work builds on two independent recent innovations. First, the parallel diffusion of multiple modalities has yielded improvements in the image generation literature. \citet{kouzelis2025boosting} and \citet{latentforcing} showed that diffusing both DINOv2 features and pixels with a single network improves performance. Second, \cite{zhou2025next} have shown that predicting higher level representations of tokens before the tokens themselves in a discrete diffusion setup improves generation quality. 

Motivated by these results, we introduce \emph{Hierarchical Continuous Diffusion Language Models} (H-CDLMs), which augment continuous diffusion language models with auxiliary modalities that represent tokens at coarser semantic granularities.
In the simplest instantiation, which we study 
in this work, we use two modalities: the original token sequence and the corresponding sequence of cluster IDs, obtained by clustering pretrained token embeddings.
We embed the cluster IDs into a continuous space and \emph{jointly} diffuse and denoise both modalities with a single model. At inference, the samplers and schedules of the two modalities can be tuned independently to optimize their interaction. 

The cluster modality is especially well suited to continuous text diffusion. Token embeddings encode rich linguistic and semantic structure: tokens that play similar roles tend to occupy nearby regions of the embedding space, and hence the same cluster. Under continuous noise, recovering the exact embedding of a token can be difficult, but identifying the region in which it lies is often still feasible. Cluster IDs supply exactly this coarse-grained signal. By indicating which region of the embedding space a token belongs to, they narrow the set of plausible tokens without committing to a specific one (Figure~\ref{fig:method_figure}). The cluster modality thus acts as a semantic anchor that steers the denoising trajectory toward the target token.
Crucially, because the two modalities are denoised in parallel, information flows bidirectionally across semantic granularities: coarse cluster predictions can guide token recovery under high noise, while token-level predictions simultaneously refine the cluster estimates.

H-CDLM adds little parameter or computational overhead when the number of clusters is small relative to the vocabulary, and can be applied directly to existing continuous DLMs. We apply it to CoBit, a state-of-the-art continuous DLM that uses bit representations to reduce output-head size. 
The resulting model, H-CoBit,  
delivers large gains in both generation quality and efficiency. 
We additionally apply H-CDLM to FLM \citep{lee2026flm}, a continuous flow matching language model, where the resulting H-FLM also provides consistent improvements over the baseline, 
highlighting the cross-paradigm applicability of our framework.

Our contributions are:
\begin{enumerate}[leftmargin=20pt, topsep=0pt, parsep=0pt]
\item We introduce H-CDLM, a framework for jointly diffusing hierarchical textual modalities in continuous DLMs, with independent samplers and noise schedules per modality.
\item We instantiate the framework using two modalities constructed from pretrained token embeddings, primarily on CoBit (yielding H-CoBit), and additionally on the flow matching model FLM (yielding H-FLM) to demonstrate applicability beyond diffusion.
\item We show that H-CoBit delivers large gains across benchmarks (Figure~\ref{fig:teaser}). At entropy matched to that of the data, it reduces generative perplexity (GenPPL) from 73.6 to 49.4 on LM1B and from 71.1 to 50.4 on OpenWebText (OWT), improvements of 24.2 and 20.7 points over CoBit that surpass even discrete DLMs, while also improving MAUVE on both datasets. It matches CoBit's performance at a substantially lower budget, and it reaches 27.4\% accuracy on GSM8K, outperforming all prior continuous diffusion and flow language models. A matched random-cluster control fails to recover these gains, confirming that they stem from meaningful cluster semantics rather than added capacity or auxiliary structure.
\item We show that H-FLM also consistently improves over FLM, confirming that the benefits of H-CDLM transfer across continuous generative paradigms.
\end{enumerate}

More broadly, H-CDLMs offer a path toward language models that denoise text at multiple levels of abstraction at once, from tokens to progressively coarser clusters and potentially to entire sentences. We view our two-level instantiation as a first step in this direction.

\section{Related work}
\label{sec:related_work}

\paragraph{Discrete Diffusion Models.}
Discrete Diffusion Models operate directly in the discrete space of tokens. They rest on a Markov chain which defines the corruption mechanism. D3PM \citep{austin2021d3pm} proposes a general framework for such processes. Corruption mechanisms range from masking to uniform token transitions. MDLM \citep{sahoo2024mdlm} first demonstrated strong generation quality for text by simplifying D3PM on masked diffusion, while Duo \citep{sahoo2025duo} showed that uniform discrete architectures can be equally good, and drew a connection to Gaussian diffusion of one hot vectors. HDLM \citep{zhou2025next} extends MDLM by building a hierarchical vocabulary of tokens, allowing for a fine grained unmasking process. TDLM \citep{wu2026rethinkingtokenpredictiontreestructured} drops HDLM's full vocabulary prediction by only predicting the next scale semantics along a hierarchical tree diffusion process -- thus considerably increasing throughput. Our work can be seen as a continuous analog of HDLM (and TDLM), although we diffuse all modalities jointly rather than unmasking clusters before tokens. Another line of work augments masked discrete diffusion with a paired continuous diffusion process: tokens are corrupted by both masking and Gaussian noise, and a single model learns to reverse both processes jointly \citep{pynadath2025candi, zhou2026coevolutionarycontinuousdiscretediffusion, continuously-augmented}.

\paragraph{Continuous Diffusion and Flow Language Models.}

Continuous Diffusion or Flow Language Models corrupt continuous representations of tokens, and train a denoiser network to reconstruct clean token embedding sequences from noisy ones. Recently, LangFlow \citep{chen2026langflow} and FLM/FMLM \citep{lee2026flm} have shown that, when carefully designed, such processes can yield similar performance to their discrete diffusion counterparts. FLM embeds tokens as one hot vectors and directly diffuses in the one hot vector space, while LangFlow exploits trained embeddings. Both models however train their denoiser to predict tokens directly and then pull the token's representation backwards into the diffusion process, entailing a large vocabulary sized output head and low throughput.
ELF \citep{hu2026elf} shows that continuous DLMs can be effective with \emph{minimal adaptation} to the discrete nature of the data, and especially without predicting actual tokens. This allows a smaller output head whose size depends on the hidden dimension rather than the vocabulary size. S-FLM \citep{deschenaux2026languagemodelinghypersphericalflows} leverages Flow Matching on the hypersphere for improved token representations, with strong performance on downstream tasks.  CoBit \citep{cobit} exploits a bit representation of tokens, previously introduced by \citet{chen2022analogbits}, alongside a carefully designed parametrization and stochastic sampling. This further reduces output head size and yields the strongest results known for continuous DLMs.

\paragraph{Multimodal Joint Diffusion.}
Multimodal joint diffusion diffuses multiple representations of a data sample in parallel, via a single denoiser network. For images, this can consist of diffusing pixels and DinoV2 \citep{oquab2024dinov} features jointly, and keeping only the generated pixels at the end of the sampling process \citep{kouzelis2025boosting, latentforcing}. This augmentation yields generation quality gains \citep{kouzelis2025boosting}, especially when coupled with an ordered scheduler which enhances the interplay between high and low level data features through the sampling process \citep{latentforcing}. This framework rests on the idea that diffusing high and low level features jointly provides a form of guidance to the denoiser. It has not, to the best of our knowledge, been extended to continuous DLMs yet. \cite{onefitsall} jointly diffuse image and text pairs, enabling unconditional text generation, but they do not evaluate this specific task.

\section{Background}
\label{sec:background}

In this section, we present a general setting for continuous diffusion and flow models for text. Our setting covers existing continuous flow and diffusion language models \citep{lee2026flm,  chen2026langflow,cobit, hu2026elf, deschenaux2026languagemodelinghypersphericalflows}, although they differ in their token representation, corruption parametrization, training objective and sampler. One exception is $\mathcal{S}$-FLM \citep{deschenaux2026languagemodelinghypersphericalflows}, whose corruption is a geodesic interpolation on the hypersphere. Details on how this setup applies are given in Appendix \ref{app:extended_background}.

\subsection{Continuous Diffusion Of Discrete Tokens}

We consider text sequences composed of $L$ tokens over a vocabulary $\mathcal{V}$. We aim to approximate the distribution $p_{\mathrm{data}} : \mathcal{V}^L \to{} [0,1]$ of a sequence $y \in \mathcal{V}^L$ via the learned model $p_{\theta}(y)$. 

Continuous diffusion first requires a continuous representation in a space $E^d$ of the vocabulary via the encoder and decoder functions $\mathcal{E}: y \to x$ and $\mathcal{D}: x \to{} y$, verifying $\mathcal{D}(\mathcal{E}(y))=y$, with $d$ the embedding dimension. Generally, we set $E=\mathbb{R}$, although recent work also explores hyperspherical embeddings, that is $E=\mathbb{S}$ \citep{deschenaux2026languagemodelinghypersphericalflows}. We write $x_0 = \mathcal{E}(y) \in \mathbb{R}^{L \times d}$, applying $\mathcal{E}$ position-wise. The diffusion process lives in $E^{L \times d}$, following the Gaussian corruption process
\begin{align}
\label{equ:forward}
    x_{\sigma} = x_0 + \sigma \epsilon,
    \qquad \epsilon \sim \mathcal{N}(\epsilon; 0, I_{Ld}),
    \qquad \sigma \in [\sigma_{\min}, \sigma_{\max}].
\end{align}
Note that, without loss of generality, we adopt the variance-exploding (EDM) convention \citep{karras2022edm}. Flow matching formulations use general parametrizations $x_t = \alpha_t x_0 + \sigma_t \epsilon$ \citep{chen2026langflow, lee2026flm}, and reduce to \eqref{equ:forward} by rescaling (see explicit correspondence in Appendix \ref{app:extended_background}). We consider a denoiser that predicts the clean data embeddings $ D_{\theta}(x_{\sigma}, \sigma) \approx \mathbb{E}[x_0 | x_{\sigma}] \in E^{L \times d}.$
\subsection{Training}

Due to the conditional expectation formulation of $D_{\theta}(x_{\sigma}, \sigma)$, the natural training objective of this model is solving the regression problem $D= \operatorname{argmin}_{\hat{D}}\mathcal{L}_{MSE}(\hat{D})$, with
\begin{align}
\label{equ:mse}
    \mathcal{L}_{\mathrm{MSE}}
    = \mathbb{E}_{x_0,\, \epsilon,\, \sigma \sim p_{\mathrm{train}}}
    \Big[
        \frac{w(\sigma)}{L} \sum_{i=1}^{L}
        \big\| D_{\theta}(x_{\sigma}, \sigma)_i - x_{0, i} \big\|^{2}
    \Big],
\end{align}
where $p_{\mathrm{train}}$ is the training noise density and $w(\sigma)$
a weighting function; both control how learning effort is allocated
across noise levels \citep{karras2022edm}. Recent work replaces or augments the regression loss with other training objectives, including cross entropy loss \citep{chen2026langflow, lee2026flm, deschenaux2026languagemodelinghypersphericalflows}. We refer to Appendix \ref{app:extended_background} for further details.

\subsection{Sampling}
\label{subsec:sampling}

We sample from this general model by backward integration of the probability-flow ODE induced by the forward process in \eqref{equ:forward}, along 
\begin{align}
\label{equ:pfode}
    \frac{dx}{d\sigma}
    = -\sigma \, \nabla_x \log p_{\sigma}(x)
    = \frac{x - D_{\theta}(x, \sigma)}{\sigma},
\end{align}
using standard solvers such as Heun or Euler on a discrete noise grid $\sigma_{\max} = \sigma_0 > \sigma_1 > \dots > \sigma_N =
\sigma_{\min}$. The final step consists in decoding position-wise via $\mathcal{D}$. The underlying discrete nature of the data requires two main adjustments that decisively improve generation quality:

\paragraph{Information-uniform noise schedules.} 
For text, the conditional information about token identities is not released uniformly with noise; it is concentrated in a narrow band of noise. Following recent continuous DLMs \citep{chen2026langflow, cobit, lee2026flm, deschenaux2026languagemodelinghypersphericalflows}, we adopt information-uniform noise schedules to focus sampling on noise regions where information about the tokens changes the most.

\paragraph{Stochastic churn.}
A second lever trades the deterministic ODE for a stochastic sampler
\citep{karras2022edm}: at step $i$, the state is churned back up to
$\hat\sigma_i = (1+\gamma)\sigma_i$ before the solver integrates down to
$\sigma_{i+1}$: $\hat{x}_i = x_i+ s_{\mathrm{noise}} \sqrt{\hat\sigma_i^{2} - \sigma_i^{2}}\epsilon$, $\epsilon \sim \mathcal{N}(\epsilon; 0,I_{Ld})$. The per-step churn rate $\gamma = S_{\mathrm{churn}}/N \geq 0$ sets how much
noise is re-injected, while
$s_{\mathrm{noise}} \geq 1$ inflates its variance \citep{karras2022edm}. Unlike $\gamma$,
$s_{\mathrm{noise}}$ is a property of the sampler rather than of the
parametrization, hence model-agnostic. Churn is essential to CoBit's generation
quality \citep{cobit}, and improves FLM (see Section \ref{subsec:h_flm}); Section~\ref{subsec:dual_noise} makes $\gamma$
\emph{per-modality}. Noise truncation \citep{deschenaux2026languagemodelinghypersphericalflows}
is a further such lever; we discuss how these intertwine with our method in
Sections~\ref{subsec:dual_noise} and Appendix~\ref{app:model_instantiation}.

\section{Method}
\label{sec:method}

In this section, we present our hierarchical framework for continuous DLMs. We define a multimodal diffusion process  (Section ~\ref{subsec:hierarchical_setup}) and detail our hierarchical modalities for text (Section~\ref{subsec:hiearchical_construc}). We then show how information-uniform noise scheduling extends to a per-modality setting (Section \ref{subsec:dual_noise}).

\begin{figure*}[t]\centering
\input{figures/method_figure_v2_body}
\input{figures/method_figure_v2_caption}
  \label{fig:method_figure}
\end{figure*}

\subsection{Hierarchical continuous diffusion}
\label{subsec:hierarchical_setup}

Following the multimodal formulation of \citet{kouzelis2025boosting}, we define a joint diffusion process over $k$ modalities, operating on vocabularies $\{ \mathcal{V}^j\}_{j=1}^{k}$. In practice, we set $k=2$ to distinguish plain tokens ($j=1$) from \emph{high level} semantic units ($j=2$), or \emph{clusters}. The diffusion process then noises continuous representations $x = \{x^j\}_{j=1}^{k}$, with $x^j \in E^{L \times d_j}$, of the data $y = \{y^j\}_{j=1}^{k}$, with $y^j \in (\mathcal{V}^j)^{L}$ and $\mathcal{V}^1 = \mathcal{V}$. The modalities are concatenated along the channel axis, so that $x \in E^{L \times d}$ with $d = \sum_{j=1}^{k} d_j$. This integrates directly into an existing continuous DLM architecture (see Section~\ref{sec:background}) by concatenating the modalities'
embeddings at each position. It does not increase sequence length.
Each modality can be corrupted at its own noise level. The forward process becomes
\begin{align}
\label{equ:forward_h}
    x_{\bm{\sigma}} = x_0 + \Sigma\, \epsilon,
    \qquad
    \epsilon \sim \mathcal{N}\big(\epsilon;0,\, I_{L d}\big),
    \qquad
    \Sigma = \operatorname{blockdiag}\big(
        \sigma_1 I_{d_1}, \dots, \sigma_k I_{d_k}
    \big) \otimes I_L,
\end{align}
where $\bm{\sigma} = (\sigma_1, \dots, \sigma_k)$ collects the
per-modality noise levels. The denoiser outputs one clean-data estimate per-modality,
$D_\theta(x_{\bm{\sigma}}, \bm{\sigma}) = (D_\theta^1, \dots,
D_\theta^k)$ with $D_\theta^j \in E^{L \times d_j}$. At training time, we minimize
\begin{align}
\label{equ:hierarchical_objective}
    \mathcal{L}_{\mathrm{MSE\text{-}H}}
    = \mathbb{E}_{x_0,\, \epsilon,\, \bm{\sigma}}
    \Big[
        \frac{1}{L} \sum_{i=1}^{L} \sum_{j=1}^{k}
        \lambda_j\, w(\sigma_j)\,
        \big\| D^{j}_{\theta}(x_{\bm{\sigma}}, \bm{\sigma})_i
             - x^{j}_{0, i} \big\|^{2}
    \Big],
\end{align}
with $\lambda_j$ per-modality loss weights. Note that expectation over time/noise is taken at the global level (see Appendix \ref{app:method_details}). At sampling time, all hierarchical levels are diffused jointly. 
At the last step, all modalities except the first (token-level) are dropped, and the tokens are decoded with $\mathcal{D}_1$.
Discarding the auxiliary modalities is principled whenever they are
deterministic functions of the tokens, which is the case for the
constructions detailed in Section~\ref{subsec:hiearchical_construc} (see Proposition \ref{prop:deterministic_latent} in appendix \ref{app:method_details}).

\subsection{Constructing hierarchical modalities for text}
\label{subsec:hiearchical_construc}

Our high-level modality provides an anchor for the diffusion of tokens. We instantiate it as \emph{token clusters},
following the semantic hierarchy of \citet{zhou2025next}. This allows us to keep our multimodal data representation intrinsically discrete. Each cluster can be attributed to a \textit{semantic} group of tokens, which can be easier to infer at high noise levels than the exact token identity. In practice, a surjective map $\Gamma : \mathcal{V} \to \mathcal{C}$ assigns each token to one of
$n = |\mathcal{C}|$ clusters, and the high-level sequence is obtained
position-wise, $y^{2, l} = \Gamma(y^{1, l})$, so that both modalities
share the sequence length $L$. We compute $\Gamma$ by $k$-means++ on pretrained token embeddings taken from LangFlow \citep{chen2026langflow} (for LM1B, OWT) and $\mathbb{S}$-FLM \citep{deschenaux2026languagemodelinghypersphericalflows} (for TinyGSM). The cluster index must itself be given a continuous encoding $\mathcal{E}_2$. The most direct way to construct $\mathcal{E}_2$ is by using the same mechanism as for $\mathcal{E}_1$, making integration to an existing DLM seamless. This involves encoding clusters via bit sequences for CoBit \citep{cobit} and through one-hot vectors for FLM \citep{lee2026flm}. This choice aims at simplicity. More complex multimodal token representations are left to future work. More generally, additional modalities ($k >2$) can be obtained by clustering at multiple granularities, yielding a deeper hierarchy. We refer to Appendix \ref{app:clustering} for further details, and to \citet{wu2026rethinkingtokenpredictiontreestructured} for a $k >2$ modalities example.

\subsection{Per-Modality Scheduling And Sampling}
\label{subsec:dual_noise}
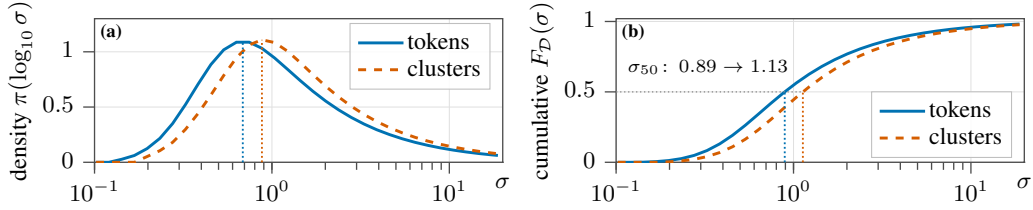
\begin{figure*}[t]
    \centering
    \setlength{\abovecaptionskip}{3pt}
\begin{tikzpicture}[
  /pgfplots/every axis/.append style={
    width=5.4cm,
    height=1.9cm,
    xmode=log,
    xmin=0.1, xmax=20,
    xtick={0.1,1,10},
    tick label style={font=\scriptsize},
    xticklabel shift=-1pt,
    typeset ticklabels with strut,
    label style={font=\scriptsize},
    legend style={font=\scriptsize, draw=none, fill=none,
                  inner sep=1pt, row sep=-2pt, nodes={inner sep=1pt}},
    ylabel shift=-3pt,
    extra x ticks={20},
    extra x tick labels={$\sigma$},
    extra x tick style={major tick length=0pt, grid=none},
  },
]
  \begin{groupplot}[
    group style={group size=2 by 1, horizontal sep=1.5cm},
  ]
    \nextgroupplot[
      paperaxis,
      ymin=0, ymax=1.3,
      ytick={0,0.5,1},
      ylabel={density $\pi(\log_{10}\sigma)$},
      legend pos=north east,
    ]
    \addplot+[color=cinfo, solid, line width=1.1pt, mark=none]
      table[row sep=\\] {
        0.102053 0\\ 0.12043 0\\ 0.142115 0.0270003\\ 0.167705 0.0614143\\
        0.197904 0.124097\\ 0.23354 0.21663\\ 0.275593 0.350713\\ 0.325219 0.521107\\
        0.383781 0.70814\\ 0.452888 0.884718\\ 0.534438 1.01762\\ 0.630674 1.08671\\
        0.744238 1.08704\\ 0.878252 1.03158\\ 1.0364 0.94357\\ 1.22302 0.842427\\
        1.44325 0.739793\\ 1.70313 0.642416\\ 2.00981 0.553862\\ 2.37171 0.475225\\
        2.79878 0.406249\\ 3.30276 0.346313\\ 3.89748 0.294632\\ 4.59929 0.250884\\
        5.42748 0.212912\\ 6.4048 0.180627\\ 7.5581 0.153356\\ 8.91907 0.129935\\
        10.5251 0.110272\\ 12.4204 0.0935189\\ 14.6569 0.0792411\\ 17.2961 0.0672152\\
        18.7889 0.061833\\
      };
    \addlegendentry{tokens}
    \addplot+[color=ccluster, dashed, line width=1.1pt, mark=none]
      table[row sep=\\] {
        0.102053 0\\ 0.12043 0\\ 0.142115 0\\ 0.167705 0\\
        0.197904 0.0519219\\ 0.23354 0.105159\\ 0.275593 0.180676\\ 0.325219 0.296876\\
        0.383781 0.442426\\ 0.452888 0.609117\\ 0.534438 0.785032\\ 0.630674 0.946056\\
        0.744238 1.06247\\ 0.878252 1.10874\\ 1.0364 1.08358\\ 1.22302 1.00607\\
        1.44325 0.903145\\ 1.70313 0.793468\\ 2.00981 0.688196\\ 2.37171 0.592391\\
        2.79878 0.507387\\ 3.30276 0.432993\\ 3.89748 0.368855\\ 4.59929 0.313545\\
        5.42748 0.266498\\ 6.4048 0.226\\ 7.5581 0.191779\\ 8.91907 0.162746\\
        10.5251 0.137988\\ 12.4204 0.116957\\ 14.6569 0.0991627\\ 17.2961 0.0841078\\
        18.7889 0.0773477\\
      };
    \addlegendentry{clusters}
    \draw[cinfo, densely dotted, line width=0.7pt] (axis cs:0.6851,0) -- (axis cs:0.6851,1.0939);
    \draw[ccluster, densely dotted, line width=0.7pt] (axis cs:0.8783,0) -- (axis cs:0.8783,1.1087);

    \nextgroupplot[
      paperaxis,
      ymin=0, ymax=1.02,
      ylabel={cumulative $F_{\mathcal{D}}(\sigma)$},
      legend pos=south east,
      ytick={0,0.5,1},
    ]
    \addplot+[color=cinfo, solid, line width=1.1pt, mark=none]
      table[row sep=\\] {
        0.102053 0\\ 0.12043 0\\ 0.142115 0.000970759\\ 0.167705 0.00482075\\
        0.197904 0.0123792\\ 0.23354 0.0262211\\ 0.275593 0.048984\\ 0.325219 0.0832611\\
        0.383781 0.130808\\ 0.452888 0.191372\\ 0.534438 0.262401\\ 0.630674 0.339639\\
        0.744238 0.418053\\ 0.878252 0.493423\\ 1.0364 0.562956\\ 1.22302 0.625394\\
        1.44325 0.680421\\ 1.70313 0.72833\\ 2.00981 0.769712\\ 2.37171 0.805252\\
        2.79878 0.835658\\ 3.30276 0.861603\\ 3.89748 0.883693\\ 4.59929 0.90249\\
        5.42748 0.918451\\ 6.4048 0.932001\\ 7.5581 0.943498\\ 8.91907 0.953247\\
        10.5251 0.961515\\ 12.4204 0.968527\\ 14.6569 0.974472\\ 17.2961 0.979511\\
        18.7889 0.981734\\
      };
    \addlegendentry{tokens}
    \addplot+[color=ccluster, dashed, line width=1.1pt, mark=none]
      table[row sep=\\] {
        0.102053 0\\ 0.12043 0\\ 0.142115 0\\ 0.167705 0\\
        0.197904 0.0031933\\ 0.23354 0.00976881\\ 0.275593 0.0211542\\ 0.325219 0.0401873\\
        0.383781 0.0692672\\ 0.452888 0.110017\\ 0.534438 0.163341\\ 0.630674 0.228561\\
        0.744238 0.303132\\ 0.878252 0.382376\\ 1.0364 0.461058\\ 1.22302 0.534972\\
        1.44325 0.601819\\ 1.70313 0.660841\\ 2.00981 0.712186\\ 2.37171 0.756459\\
        2.79878 0.794418\\ 3.30276 0.826834\\ 3.89748 0.854468\\ 4.59929 0.877973\\
        5.42748 0.897945\\ 6.4048 0.914903\\ 7.5581 0.929292\\ 8.91907 0.941493\\
        10.5251 0.951843\\ 12.4204 0.960612\\ 14.6569 0.96805\\ 17.2961 0.974359\\
        18.7889 0.97714\\
      };
    \addlegendentry{clusters}
    \draw[densely dotted, black!45, line width=0.5pt] (axis cs:0.1,0.5) -- (axis cs:1.1291,0.5);
    \draw[cinfo, densely dotted, line width=0.7pt] (axis cs:0.8923,0) -- (axis cs:0.8923,0.5);
    \draw[ccluster, densely dotted, line width=0.7pt] (axis cs:1.1291,0) -- (axis cs:1.1291,0.5);
    \node[font=\scriptsize, anchor=north west, inner sep=1.5pt] at (axis cs:0.11,0.76) {$\sigma_{50}\!:\ 0.89 \rightarrow 1.13$};
  \end{groupplot}
  \node[font=\scriptsize\bfseries, anchor=north west, inner sep=2pt]
    at (group c1r1.north west) {(a)};
  \node[font=\scriptsize\bfseries, anchor=north west, inner sep=2pt]
    at (group c2r1.north west) {(b)};
\end{tikzpicture}
\caption{Token vs.\ cluster information-uniform noise schedules for H-CoBit:
  (a) density over log-noise levels; (b) cumulative distributions, with the median
  shifting from $\sigma_{50}=0.89$ to $1.13$.}
\label{fig:noise-schedules}
\end{figure*}
A key opportunity of joint diffusion is to adapt the scheduler and sampler \emph{per-modality}, so as to enhance the interplay between modalities throughout the denoising process, and to adapt the noise schedule and sampler to each modality's underlying data structure. 

\paragraph{Per-modality scheduler.} In multimodal joint image diffusion, \citet{latentforcing} for instance order the modalities: the high level modality is kept cleaner than the pixels via per-modality schedulers, so that resolved semantics guide pixel denoising. In image diffusion, this requires setting an
explicit ordering of the modalities' schedules. Our framework offers a simpler alternative: training a \emph{per-modality} information-uniform noise schedule. For each modality $j$ and sampling step $i$, $\sigma_{i,j}$ is set so that it yields a grid on which each step reveals approximately equal conditional information about the underlying data (see Section \ref{subsec:sampling} and Appendix \ref{app:per_modality_schedul} for details on how to train $\sigma_{i,j}$).  Figure \ref{fig:noise-schedules} displays such trained schedules for 2 modalities. Cluster density is right-shifted: cluster identity is a coarser variable, so information about it survives larger noise and its entropy is released at higher $\sigma$. Figure \ref{fig:noise-schedules} shows that each modality has a different underlying structure. This motivates adapting the scheduler and sampler \emph{per-modality}. In our experiments, we set our scheduler per-modality by default, and ablate the possibility of using a single shared scheduler (see Table \ref{tab:dual_schedul}).

\paragraph{Per-modality sampler (PMS).} One direct way to adapt the sampler per-modality is to leverage stochastic sampling (see Section \ref{subsec:sampling}). We enable two key mechanisms: (i) adapt the churn rate per-modality, so that $\hat\sigma_{i,j}=(1+\gamma_i^j)\sigma_{i,j}$, $\hat{x}_i^j = x_i^j+ s_{\mathrm{noise}} \sqrt{\hat\sigma_{i,j}^{2} - \sigma_{i,j}^{2}}\epsilon$ for $\epsilon \sim \mathcal{N}(0,I)$. (ii) churn clusters on a predefined window $\mathcal{W}^2$. We set it to the first ($\mathcal{W}_{\text{early}}^2$) or last ($\mathcal{W}_{\text{late}}^2$) half of sampling steps. 
Combined, these mechanisms yield clear performance improvements in H-CDLM.
We provide a comprehensive evaluation of stochastic sampling per-modality in its H-CoBit instantiation in Appendix \ref{app:multimodal_churn}. By construction, any sampler mechanism or parameter can be tuned per-modality to improve the interplay between modalities. We leave further tuning to future work.
\section{Experiments}
\label{sec:experiments}

We mainly evaluate our framework on CoBit \citep{cobit}, a state-of-the-art (SOTA) continuous DLM in terms of downstream task accuracy and GenPPL frontier. We first present our experimental setup (Section \ref{subsec:exp_setup}), then our main training and evaluation results for H-CoBit on standard text generation and downstream tasks (Section \ref{subsec:main_results}), and our ablations (Section \ref{subsec:ablations}). We finally explore cross-paradigm transferability of H-CDLM by applying it to FLM (Section \ref{subsec:h_flm}).

\subsection{Experimental Setup}
\label{subsec:exp_setup}

\paragraph{Training.}

Following \cite{chen2026langflow, sahoo2024mdlm, sahoo2025duo, deschenaux2026languagemodelinghypersphericalflows, cobit, hu2026elf}, we train on LM1B \citep{lm1b} using 128-token sequences and the BERT-base-uncased tokenizer ($V=30522$), on OWT \citep{owt} using 1024-token sequences with the GPT-2 tokenizer ($V=50257$), and on TinyGSM \citep{liu2023tinygsm} using 512-token sequences with the SmolLM-135M tokenizer ($V=49153$).
Our denoiser is a 12-layer, 12-head Diffusion Transformer with hidden size 768, resulting in 134.6M and 135M parameters for LM1B and OWT respectively (parameter overhead of  approximately +1$\%$ compared to the CoBit baseline). All other training parameters are set equal to \cite{cobit} for H-CoBit, and \citet{lee2026flm} for H-FLM.

\paragraph{Evaluation.}
We evaluate by measuring GPT-2 Large generative perplexity (GenPPL) - entropy (H) frontiers, following \cite{sahoo2024mdlm, chen2026langflow, deschenaux2026languagemodelinghypersphericalflows, hu2026elf, cobit}. We also compute MAUVE on LM1B and OWT, and GSM8K \citep{gsm8k} accuracy for our model trained on TinyGSM. MAUVE and GenPPL are evaluated on 1024 generated sequences. All results are reported across a full sweep over Number of Function Evaluations (NFE) and sampling parameters, especially churn which is essential for CoBit. Per-modality loss weighting coefficients $\lambda_j$ are tuned for each model/dataset.
For CoBit, we keep the original paper's inference hyperparameters: global stochastic sampling noise is $s_{\text{noise}}=1.003$ on LM1B, OWT and GSM8K. We ablate $s_{\text{noise}}=1.05$, for fair comparison with our per-modality samplers on H-CoBit. We refer to Appendix \ref{app:exp_details} for further details.

\subsection{Main results}
\label{subsec:main_results}
\begin{figure*}[t]\centering
\input{figures/genppl_entropy_frontier_3panel_caption}
\input{figures/genppl_entropy_frontier_3panel_body}
  \label{fig:frontier}
  \vspace{-8pt}
\end{figure*}
Figure \ref{fig:frontier} displays improvements of our model H-CoBit in the GenPPL-entropy frontiers on LM1B and OWT, shown by right shifts of the frontier. Table \ref{tab:main_table} shows GenPPL levels at given entropies. On LM1B, H-CoBit matches dataset entropy at GenPPL 57.3, more than 20 points better than CoBit-S, at only 60\% of the training steps. On OWT at matched 750k steps, H-CoBit reaches GenPPL 67.8 at dataset entropy beating CoBit-S by 15.4 points. Adding per-modality churn further improves H-CoBit's advantage: on OWT, combined with $s_{noise}=1.05$, H-CoBit reaches Gen PPL 50.4, better than CoBit-S at same $s_{noise}$ (71.1), and than any DLM baseline at dataset entropy.
H-CoBit also improves MAUVE at dataset entropy, reaching 0.903 against 0.861 on LM1B, and 0.580 against 0.530 on OWT against CoBit (see Table \ref{tab:MAUVE}, Appendix \ref{app:add_exp}). We note H-CoBit incurs only a negligible worst case computational overhead (+4.2\%) while matching CoBit’s performance at much lower budget (see Figure~\ref{fig:teaser}, Table \ref{tab:speed}).
\begin{table*}[t]\centering
\caption{\textbf{Generation quality, diversity at one operating point.} Every row is one generation scored by GenPPL and entropy $H$. All reported continuous DLMs have $\approx$130M parameters (except ELF, 342M). 
H-CoBit is our hierarchical continuous DLM framework on CoBit-S.
H-CoBit uses cluster-loss weights $\lambda_2{=}0.27$ / $0.1$ for (LM1B / OWT). 
We report GenPPL at matched dataset entropy when reached by the model. Per-Modality Sampler (PMS) samples with tuned churn per-modality: for H-CoBit, LM1B uses $\mathcal{W}^2_{\text{late}}$, $\gamma_2=0.4$ and OWT uses $\mathcal{W}_{\text{late}}$, $\gamma_2=0.3$ + global $s_{noise}=1.05$ (compared to Cobit-S, $s_{noise}=1.05$). 
The bold is best at matched H. Baselines are taken from original papers, \citet{franca2026hackinggenerativeperplexityunconditional} or \citet{pynadath2026generativefrontiersevaluationmatters}, except for S-FLM and CoBit which we re-evaluate.}

\begingroup\small
\setlength{\tabcolsep}{3pt}
\newcommand{\tgrp}[1]{\multicolumn{9}{@{}l}{\textit{#1}}}
\newcommand{\tsub}[1]{\multicolumn{9}{@{}l}{\textit{#1}}}
\newcommand{\tind}{\hspace{0.8em}}
\definecolor{oursbg}{gray}{0.90}%
\newcommand{\tours}{\rowcolor{oursbg}}
\begin{tabular}{@{}l ccrr >{\hspace{6pt}}c crr@{}}
\toprule
 & \multicolumn{4}{c}{\textbf{LM1B} ($L{=}128$)}
 & \multicolumn{4}{c}{\textbf{OpenWebText} ($L{=}1024$)} \\
\cmidrule(lr){2-5}\cmidrule(l){6-9}
\textbf{Model}
 & Steps & NFE & $H\!\uparrow$ & GenPPL$\downarrow$
 & Steps & NFE & $H\!\uparrow$ & GenPPL$\downarrow$ \\
\midrule
Real text (test split) & -- & -- & 4.34 & 53.1 & -- & -- & 5.46 & 14.9 \\
\midrule
\tsub{\tind Autoregressive} \\
\tind\tind AR baseline \citep{franca2026hackinggenerativeperplexityunconditional}
 & & & & & -- & -- & 5.61 & 40.2 \\
\addlinespace[2pt]
\tsub{\tind Discrete diffusion} \\
\tind\tind SEDD~\citep{lou2024discrete}
 & & & & & -- & -- & 5.65 & 125.7 \\
\tind\tind Duo~\citep{sahoo2025duo}
 & 1M & -- & 4.28 & 98.6 & 1M & 128 & 5.46 & 63.3 \\
\tind\tind MDLM~\citep{sahoo2024mdlm}
 & 1M & 1024 & 4.32 & 109.2 & 1M & 128 & 5.46 & 66.2 \\
\addlinespace[2pt]
\tsub{\tind Hybrid discrete--continuous} \\
\tind\tind CANDI \citep{pynadath2025candi}
 & -- & -- & 4.32 & 119.9 & 1M & 128 & 5.46& 71.0\\
 \tind\tind CADD \citep{continuously-augmented}
 & -- & -- & -- & -- & 1M & 128 & 5.48& 55.3\\
\addlinespace[2pt]
\tsub{\tind Continuous diffusion / flow} \\
\tind\tind LangFlow~\citep{chen2026langflow}
 & 1M & 128 & 4.31 & 92.2 & 1M & 128 & 5.43 & 60.1 \\
\tind\tind ELF-M (342M)~\citep{hu2026elf}
 & & & & & 95k & 64 & 5.44 & 62.5 \\
\tind\tind FLM~\citep{lee2026flm}
 & 1M & 1024 & 4.28 & 119.1 & 1M & 1024 & 5.40 & 96.8 \\
\tind\tind S-FLM~\citep{deschenaux2026languagemodelinghypersphericalflows}
 & & & & & 1M & 256 & 5.46 & 108.2 \\
\midrule
\tgrp{CoBit~\citep{cobit} vs.\ H-CoBit (ours)} \\
\tind CoBit-S
 & 1M & 128 & 4.34 & 85.2 & 750k & 256 & 5.46 & 83.2 \\
\tours\tind H-CoBit
 & 600k & 256 & 4.34 & 57.3 & 750k & 128 & 5.45 & 67.8 \\
\addlinespace[3pt]
\tind CoBit-S, $s_{\mathrm{noise}}{=}1.05$
 & 1M & 256& 4.34& 73.6 & 750k & 256 & 5.46 & 71.1 \\
\tours\tind H-CoBit $+$ PMS
 & 600k & 512 & 4.34 & \textbf{49.4} & 750k & 256 & 5.46 & \textbf{50.4} \\
\bottomrule
\end{tabular}
\endgroup

  \label{tab:main_table}
\end{table*}

 Table \ref{tab:gsm8k} displays H-CoBit's accuracy on GSM8K. On these 1319 mathematical reasoning problems, H-CoBit beats all other continuous DLMs, and narrows the gap to discrete DLMs.
\begin{table*}[t]\centering
\caption{\textbf{Mathematical reasoning accuracy.} GSM8K accuracy at matched training budget. CoBit/H-CoBit are reported at their best churn. H-CoBit ($\lambda_2=0.3$) beats CoBit by +2.1pts at half the NFE, +4pts at matched NFE. Baselines are from \citet{deschenaux2026languagemodelinghypersphericalflows, cobit}.}

\begingroup\small
\setlength{\tabcolsep}{6pt}
\begin{tabular}{@{}llrr@{}}
\toprule
\textbf{Model} & \textbf{Family} & \textbf{NFE} & \textbf{Acc.\ (\%)} \\
\midrule
\multicolumn{4}{@{}l}{\itshape Autoregressive (reference ceiling)} \\
AR (greedy) & autoregressive LM & -- & 63.3 \\
AR (sampling) & autoregressive LM & -- & 53.9 \\
\midrule
\multicolumn{4}{@{}l}{\itshape Discrete DLMs} \\
MDLM (low-T) & masked discrete diffusion & 1024 & $\sim$33\\
Duo (low-T) & discrete diffusion & 1024 & $\sim$36 \\
\midrule
\multicolumn{4}{@{}l}{\itshape Continuous DLMs} \\
FLM & continuous flow & 1024 & 0.3 \\
S-FLM (+top-$k{=}1$) & hyperspherical flow & 1024 & 18.0 \\
CoBit-S (no cluster bits) & bitstream diffusion & 512 & 23.4 \\
CoBit-S (no cluster bits) & bitstream diffusion & 1024 & 25.4 \\
H-CoBit ($n{=}16$) \textbf{(ours)} & bitstream diffusion & 512 & \textbf{27.4} \\
\bottomrule
\end{tabular}
\endgroup

  \label{tab:gsm8k}
  \vspace{-4pt}
\end{table*}
\subsection{Ablations}
\label{subsec:ablations}

In this section, we discuss ablations of our framework on H-CoBit. Due to limited compute, when retraining is required to conduct an ablation, we restrict ourselves to LM1B at 100k training steps.

\paragraph{Number of clusters and cluster loss weight.} Figure \ref{fig:k_lmabda_ablation}(b,c) studies the optimal number of clusters $n$ and cluster loss weight $\lambda_2$. Both panels (c) and (b) show that $n=16$ is optimal: it is the only $n$ to beat the baseline in all values of $\lambda_2$. Here, the multimodal process is beneficial at a very low number of clusters compared to vocabulary size ($|V|=30522$ for LM1B). We hypothesize that this is closely tied to the clustering. At $n=16$, the optimal cluster loss weight is $\lambda_2^{\text{opt}}=5\cdot10^{-2}$ (we use $\lambda_2=\frac{4}{15}$, implying LM1B performance could be further improved). This is likely highly dependent on $n$ and the dataset, thus requiring careful tuning for each new task.
\begin{figure*}[h]\centering
\input{figures/cluster_k_and_weight_ablation_body}
\vspace{-15pt}
\input{figures/cluster_k_and_weight_ablation_caption}
  \label{fig:k_lmabda_ablation}
  \vspace{-12pt}
\end{figure*}
\paragraph{Semantic vs. random cluster structure.} In order to verify that the semantic structure of our \emph{high level modality} is meaningful, we train and evaluate H-CoBit with \emph{random} cluster assignments. This means that a token's cluster is random, not dictated by a clustering of pretrained token embeddings.
Figure \ref{fig:k_lmabda_ablation}(a) compares the performance of the CoBit baseline to H-CoBit with random or with semantic clusters. It clearly shows that the random clusters model's frontier is intertwined with that of the baseline, while the semantic cluster model yields a right shift.
Hence, cluster semantics matter. It is insufficient to provide \emph{high level} categories; these categories should bear meaning. We also show in Appendix \ref{subsec:token_cluster agreement} that the cluster modality carries information about the final token beyond that of the token modality, further stressing the benefit of our joint multimodal diffusion framework.

\paragraph{Per-modality scheduling and sampling.} For H-CoBit, we tune two per-modality sampling mechanisms (see Section \ref{subsec:dual_noise}). \textbf{(1)} The \emph{per-modality} information uniform schedule. Table \ref{tab:dual_schedul} compares this with a single shared schedule, showing the per-modality schedule is slightly better in MAUVE and GenPPL at dataset entropy, but worse at lower entropies. 
Neither schedule thus clearly dominates the other.
\textbf{(2)} The stochastic churn level is scheduled and scaled differently for the cluster modality. We adopt two churning strategies: churning the cluster modality (via $\gamma_2$) at a different level, and churning it with a phase. Churning it only late, at the low half of noise, achieves -7.9/-17.4 GenPPL gains on the shared churn control on LM1B/OWT (see Figure \ref{fig:frontier}, Table \ref{tab:dual_churn}).

\begin{table*}[h]\centering
\caption{\textbf{Per-modality vs single schedule, H-CoBit.} LM1B $n{=}16$. MAUVE is the best reached.}

\begingroup\small
\setlength{\tabcolsep}{6pt}
\begin{tabular}{@{}lrrrrrrr@{}}
\toprule
 &  &  & \multicolumn{4}{c}{\textbf{GenPPL at token entropy $H$}} &  \\
\cmidrule(lr){4-7}
\textbf{Arm} & \textbf{$H{\approx}4.28$} & \textbf{$H{\approx}4.30$} & \textbf{$H{\approx}4.32$} & \textbf{$H{\approx}4.34$} & \textbf{MAUVE} \\
\midrule
dual (per-modality schedule) & 46.4 & 46.9 & \textbf{51.6} & \textbf{57.3} & \textbf{0.973} \\
single (shared schedule) & \textbf{44.6} & \textbf{45.1} & 52.8 & 57.4 & 0.965 \\
\bottomrule
\end{tabular}
\endgroup

  \label{tab:dual_schedul}
  \vspace{-8pt}
\end{table*}

\subsection{Cross paradigm transferability: H-FLM}
\label{subsec:h_flm}

In this section, we assess whether the benefits of H-CDLM transfer beyond diffusion by applying it to the Flow Language Model (FLM) \citep{lee2026flm}, a continuous flow matching model. We show that the resulting H-FLM also improves over FLM, confirming that our framework is not tied to a specific continuous generative paradigm.

\paragraph{H-FLM parametrization.} We parametrize H-FLM as suggested in Section \ref{sec:method}: the cluster modality is represented as a $n$-dimensional one-hot vector which is corrupted following its own error rate based schedule. All other training parameters are identical to \citet{lee2026flm}.

\paragraph{H-FLM results on LM1B.} Figure \ref{fig:flm_lm1B} compares the GenPPL entropy frontiers of FLM and H-FLM, using  stochastic sampling on LM1B at 1M training steps. The stochastic samplers used for FLM and H-FLM are defined after a careful sweep over churn levels. Stochastic sampling improves both FLM and H-FLM, a phenomenon that is not reported by FLM's authors \citep{lee2026flm}. H-FLM demonstrates a clear right shift with respect to FLM under stochastic sampling, and a 11 GenPPL points improvement at dataset entropy for NFE=512. We stress that the winning stochastic sampling parameters for H-FLM are similar to those of H-CoBit: late half window cluster churn at $\gamma_2=0.3$. This shows that: \textit{(i)} joint multimodal diffusion/flow matching via H-CDLM is beneficial over different CDLM architectures, \textit{(ii)} H-CDLMs seem to require stochastic sampling, and churning cluster latents late is beneficial. 

\begin{figure*}[h!]\centering
\begin{tikzpicture}
  \begin{groupplot}[
    group style={group size=3 by 1, horizontal sep=0.9cm, vertical sep=1.15cm},
  ]
    \nextgroupplot[
      paperaxis,
      width=3.6023307845579073cm,
      height=2.7cm,
      xmin=4,
      xmax=4.46,
      ymin=20,
      ymax=150,
      ylabel={GenPPL (GPT-2 large)},
      scaled x ticks=false,
      scaled y ticks=false,
      xtick={4.0,4.2,4.4},
      xticklabel style={/pgf/number format/.cd, fixed, zerofill, precision=1},
      ytick={50,100,150},
      label style={font=\scriptsize},
      tick label style={font=\tiny},
    ]
    \addplot+[color=cdata, densely dashed, line width=0.8pt, mark=none, forget plot]
      table[row sep=\\] {
        4.34 20\\
        4.34 200\\
      };
    \addplot+[color=ccobit, solid, line width=1.1pt, mark=square*, mark size=1.1pt, only marks, mark size=1.5pt, mark options={solid, fill=ccobit, draw=ccobit}, only marks, draw opacity=0.25, fill opacity=0.25, forget plot]
      table[row sep=\\] {
        3.69428 135.444\\
        1.3693 7.5343\\
      };
    \addplot+[color=ccobit, solid, line width=1.1pt, mark=square*, mark size=1.1pt, line width=1.2pt, mark size=1.5pt, mark options={solid, fill=ccobit, draw=ccobit}, forget plot]
      table[row sep=\\] {
        4.06475 55.7584\\
        4.12129 59.7728\\
        4.17388 65.1188\\
        4.23567 75.2647\\
        4.29847 91.4646\\
        4.36874 123.818\\
        4.41966 201.628\\
      };
    \addplot+[color=cchurnwin, solid, line width=1.3pt, mark=square*, mark size=1.1pt, only marks, mark size=1.5pt, mark options={solid, fill=cchurnwin, draw=cchurnwin}, only marks, draw opacity=0.25, fill opacity=0.25, forget plot]
      table[row sep=\\] {
        4.40154 184.868\\
        3.53595 110.639\\
        1.58428 10.9603\\
        0.526137 2.52346\\
      };
    \addplot+[color=cchurnwin, solid, line width=1.3pt, mark=square*, mark size=1.1pt, line width=1.2pt, mark size=1.5pt, mark options={solid, fill=cchurnwin, draw=cchurnwin}, forget plot]
      table[row sep=\\] {
        4.10879 43.545\\
        4.13078 45.3641\\
        4.15059 46.9716\\
        4.17765 51.122\\
        4.20902 55.8238\\
        4.25174 64.5574\\
        4.31006 83.4705\\
        4.34301 96.328\\
        4.38421 117.797\\
        4.4215 149.396\\
      };
    \node[anchor=north west, draw=black!25, line width=0.4pt, fill=white, fill opacity=0.9, text opacity=1, inner sep=2pt, font=\tiny] at (rel axis cs:0.025,0.975) {\begin{tabular}{@{}l@{~}l@{}}\tikz[baseline=-0.5ex]{\draw[ccobit, solid, line width=1.1pt] (0pt,0pt)--(13pt,0pt);} & FLM \\
\tikz[baseline=-0.5ex]{\draw[cchurnwin, solid, line width=1.3pt] (0pt,0pt)--(13pt,0pt);} & H-FLM \\
\tikz[baseline=-0.5ex]{\draw[cdata, densely dashed, line width=0.8pt] (0pt,0pt)--(13pt,0pt);} & data entropy\end{tabular}};
    \nextgroupplot[
      paperaxis,
      width=3.6023307845579073cm,
      height=2.7cm,
      xmin=4,
      xmax=4.46,
      ymin=20,
      ymax=150,
      scaled x ticks=false,
      scaled y ticks=false,
      xtick={4.0,4.2,4.4},
      xticklabel style={/pgf/number format/.cd, fixed, zerofill, precision=1},
      ytick={50,100,150},
      label style={font=\scriptsize},
      tick label style={font=\tiny},
    ]
    \addplot+[color=cdata, densely dashed, line width=0.8pt, mark=none, forget plot]
      table[row sep=\\] {
        4.34 20\\
        4.34 200\\
      };
    \addplot+[color=ccobit, solid, line width=1.1pt, mark=diamond*, mark size=1.1pt, only marks, mark size=1.5pt, mark options={solid, fill=ccobit, draw=ccobit}, only marks, draw opacity=0.25, fill opacity=0.25, forget plot]
      table[row sep=\\] {
        4.07377 167.566\\
        2.90849 48.3803\\
        1.30085 6.79468\\
      };
    \addplot+[color=ccobit, solid, line width=1.1pt, mark=diamond*, mark size=1.1pt, line width=1.2pt, mark size=1.5pt, mark options={solid, fill=ccobit, draw=ccobit}, forget plot]
      table[row sep=\\] {
        4.10529 49.4543\\
        4.1346 50.9749\\
        4.17002 54.7991\\
        4.20415 60.1357\\
        4.24814 68.5061\\
        4.30659 86.1933\\
        4.33385 99.8397\\
        4.3408 102.531\\
        4.35329 108.245\\
        4.38173 131.366\\
        4.40337 178.221\\
      };
    \addplot+[color=cchurnwin, solid, line width=1.3pt, mark=diamond*, mark size=1.1pt, only marks, mark size=1.5pt, mark options={solid, fill=cchurnwin, draw=cchurnwin}, only marks, draw opacity=0.25, fill opacity=0.25, forget plot]
      table[row sep=\\] {
        2.95369 36.7565\\
        0.458574 2.17921\\
        0.12894 1.42946\\
      };
    \addplot+[color=cchurnwin, solid, line width=1.3pt, mark=diamond*, mark size=1.1pt, line width=1.2pt, mark size=1.5pt, mark options={solid, fill=cchurnwin, draw=cchurnwin}, forget plot]
      table[row sep=\\] {
        4.0758 36.7025\\
        4.0919 37.6219\\
        4.12397 39.2745\\
        4.14807 41.1841\\
        4.17127 43.9764\\
        4.20838 49.2978\\
        4.25805 58.3362\\
        4.28319 64.1963\\
        4.31461 75.0321\\
        4.35157 91.0951\\
        4.36102 110.307\\
      };
    \nextgroupplot[
      paperaxis,
      width=3.6023307845579073cm,
      height=2.7cm,
      xmin=4,
      xmax=4.46,
      ymin=20,
      ymax=100,
      scaled x ticks=false,
      scaled y ticks=false,
      xtick={4.0,4.2,4.4},
      xticklabel style={/pgf/number format/.cd, fixed, zerofill, precision=1},
      ytick={25,50,100},
      label style={font=\scriptsize},
      tick label style={font=\tiny},
      after end axis/.append code={\node[anchor=north, font=\scriptsize] at (rel axis cs:-0.7498,-0.18518518518518517) {token unigram entropy $H$ (nats)};},
    ]
    \addplot+[color=cdata, densely dashed, line width=0.8pt, mark=none, forget plot]
      table[row sep=\\] {
        4.34 20\\
        4.34 200\\
      };
    \addplot+[color=ccobit, solid, line width=1.1pt, mark=pentagon*, mark size=1.1pt, only marks, mark size=1.5pt, mark options={solid, fill=ccobit, draw=ccobit}, only marks, draw opacity=0.25, fill opacity=0.25, forget plot]
      table[row sep=\\] {
        4.11464 154.026\\
        0.931182 4.24081\\
      };
    \addplot+[color=ccobit, solid, line width=1.1pt, mark=pentagon*, mark size=1.1pt, line width=1.2pt, mark size=1.5pt, mark options={solid, fill=ccobit, draw=ccobit}, forget plot]
      table[row sep=\\] {
        4.14101 47.5929\\
        4.15888 49.5864\\
        4.17116 50.8055\\
        4.19989 54.7269\\
        4.23209 61.661\\
        4.28381 74.4524\\
        4.35388 111.068\\
      };
    \addplot+[color=cchurnwin, solid, line width=1.3pt, mark=pentagon*, mark size=1.1pt, only marks, mark size=1.5pt, mark options={solid, fill=cchurnwin, draw=cchurnwin}, only marks, draw opacity=0.25, fill opacity=0.25, forget plot]
      table[row sep=\\] {
        4.32979 80.0268\\
        4.32534 82.8438\\
        4.30348 83.9225\\
        4.27322 85.2076\\
        2.33353 16.4212\\
        0.185267 1.54031\\
        0.117602 1.39988\\
      };
    \addplot+[color=cchurnwin, solid, line width=1.3pt, mark=pentagon*, mark size=1.1pt, line width=1.2pt, mark size=1.5pt, mark options={solid, fill=cchurnwin, draw=cchurnwin}, forget plot]
      table[row sep=\\] {
        4.02924 34.3482\\
        4.06965 34.5788\\
        4.09447 35.8563\\
        4.12907 37.7484\\
        4.15931 40.1425\\
        4.18999 44.0512\\
        4.23518 51.5596\\
        4.25527 56.0913\\
        4.28567 63.6388\\
        4.31605 75.8338\\
        4.33181 78.7755\\
      };
  \end{groupplot}
  \node[panellabel] at (group c1r1.north west) {(a) NFE 128};
  \node[panellabel] at (group c2r1.north west) {(b) NFE 512};
  \node[panellabel] at (group c3r1.north west) {(c) NFE 1024};
\end{tikzpicture}
\input{figures/flm_frontier_scale_caption}
  \label{fig:flm_lm1B}
\end{figure*}

\section{Conclusion}

We have proposed a new multimodal joint diffusion framework for text with minimal parameter overhead. To our knowledge, this is the first work to integrate hierarchical representations of tokens into continuous diffusion processes for text. It integrates innovations from the discrete hierarchical diffusion and the multimodal joint image diffusion literature. This multimodal setup yields consistent gains in generation quality and downstream task accuracy when applied to CoBit, beating recent DLMs in generative quality and continuous DLMs on a mathematical reasoning task. Similar performance gains are observed when applying our framework to FLM, highlighting its cross-model applicability. Our simple instantiation's large empirical gains point to a broader direction: 
generating text 
jointly across multiple representations of a sequence, spanning sub-word units to sentence-level semantics.

\section*{Acknowledgments}
Computational resources for this work were provided by CINECA's Leonardo Cluster, as part of a EuroHPC Development Call Access. We also thank the Archimedes Unit of the Athena Research Center for making this collaboration possible

\section*{AI Use Statement}
\label{sec:ai_use}

In this work, we used generative AI tools to create or edit software code, to edit figures and tables, to review grammar and formulation mistakes, and to polish text. We have not used generative AI tools for research directions, to formulate proofs, propose hypotheses, or interpret results, and the rest of the uses of AI required to disclose are not applicable for this work. We have reviewed all AI-assisted work. We take responsibility for the final content of this work, including text, claims or artifacts produced with the aid of generative AI.

\section*{Ethics statement}
Generative models carry a substantial risk of misuse. Their application can lead to various negative societal impacts, most notably the spread of disinformation. Enhancements in generative performance, as achieved by our method, may further increase the realism of generated content, potentially making disinformation even more convincing.

\section*{Reproducibility statement}
We provide implementation details, dataset descriptions and preprocessing, hyperparameters and evaluation metrics for reproducing H-CoBit results in Section\ref{sec:experiments} and Appendix \ref{app:exp_details}. All our implementations are based on open source code. Code and model checkpoints will be made publicly available upon acceptance.

\bibliography{literature}

@inproceedings{karras2022edm,
  title={Elucidating the Design Space of Diffusion-Based Generative Models},
  author={Karras, Tero and Aittala, Miika and Aila, Timo and Laine, Samuli},
  booktitle={Advances in Neural Information Processing Systems},
  year={2022}
}

@inproceedings{austin2021d3pm,
  title={Structured Denoising Diffusion Models in Discrete State-Spaces},
  author={Austin, Jacob and Johnson, Daniel D. and Ho, Jonathan and Tarlow, Daniel and van den Berg, Rianne},
  booktitle={Advances in Neural Information Processing Systems},
  year={2021}
}

@inproceedings{chen2022analogbits,
  title={Analog Bits: Generating Discrete Data using Diffusion Models with Self-Conditioning},
  author={Chen, Ting and Zhang, Ruixiang and Hinton, Geoffrey},
  booktitle={International Conference on Learning Representations},
  year={2023},
  url={https://openreview.net/forum?id=3itjR9QxFw}
}

@inproceedings{
sahoo2024mdlm,
title={Simple and Effective Masked Diffusion Language Models},
author={Subham Sekhar Sahoo and Marianne Arriola and Aaron Gokaslan and Edgar Mariano Marroquin and Alexander M Rush and Yair Schiff and Justin T Chiu and Volodymyr Kuleshov},
booktitle={The Thirty-eighth Annual Conference on Neural Information Processing Systems},
year={2024},
url={https://openreview.net/forum?id=L4uaAR4ArM}
}

@inproceedings{
  sahoo2025duo,
  title={The Diffusion Duality},
  author={Subham Sekhar Sahoo and Justin Deschenaux and Aaron Gokaslan and Guanghan Wang and Justin T Chiu and Volodymyr Kuleshov},
  booktitle={Forty-second International Conference on Machine Learning},
  year={2025},
  url={https://openreview.net/forum?id=9P9Y8FOSOk}
}

@inproceedings{
pynadath2025candi,
title={{CANDI}: Hybrid Discrete-Continuous Diffusion Models},
author={Patrick Pynadath and Jiaxin Shi and Ruqi Zhang},
booktitle={Forty-third International Conference on Machine Learning},
year={2026},
url={https://openreview.net/forum?id=03KQGUvw1e}
}

@misc{lee2026flm,
  title={Flow Map Language Models: One-step Language Modeling via Continuous Denoising},
  author={Lee, Chanhyuk and Yoo, Jaehoon and Agarwal, Manan and Shah, Sheel and Huang, Jerry and Raghunathan, Aditi and Hong, Seunghoon and Boffi, Nicholas M. and Kim, Jinwoo},
  year={2026},
  eprint={2602.16813},
  archivePrefix={arXiv},
  primaryClass={cs.CL},
  url={https://arxiv.org/abs/2602.16813}
}

@misc{chen2026langflow,
  title={LangFlow: Continuous Diffusion Rivals Discrete in Language Modeling},
  author={Chen, Yuxin and Liang, Chumeng and Sui, Hangke and Guo, Ruihan and Cheng, Chaoran and You, Jiaxuan and Liu, Ge},
  year={2026},
  eprint={2604.11748},
  archivePrefix={arXiv},
  primaryClass={cs.LG},
  url={https://arxiv.org/abs/2604.11748}
}

@misc{hu2026elf,
  title={ELF: Embedded Language Flows},
  author={Hu, Keya and Qiu, Linlu and Lu, Yiyang and Zhao, Hanhong and Li, Tianhong and Kim, Yoon and Andreas, Jacob and He, Kaiming},
  year={2026},
  eprint={2605.10938},
  archivePrefix={arXiv},
  primaryClass={cs.CL},
  url={https://arxiv.org/abs/2605.10938}
}

@article{raya2026infonoise,
  title   = {Information-Guided Noise Allocation for Efficient Diffusion Training},
  author  = {Raya, Gabriel and Nguyen, Bac and Batzolis, Georgios and Takida, Yuhta and Stancevic, Dejan and Murata, Naoki and Lai, Chieh-Hsin and Mitsufuji, Yuki and Ambrogioni, Luca},
  eprint={2602.18647},
  archivePrefix={arXiv},
primaryClass={cs.CL},
  year    = {2026},
  url     = {https://arxiv.org/abs/2602.18647}
}

@inproceedings{stancevic2025entropic,
  title     = {Entropic Time Schedulers for Generative Diffusion Models},
  author    = {Stancevic, Dejan and Handke, Florian and Ambrogioni, Luca},
  booktitle = {Advances in Neural Information Processing Systems},
  year      = {2025},
  url       = {https://openreview.net/forum?id=EfDIApcjgI}
}

@inproceedings{lou2024discrete,
  title     = {Discrete Diffusion Modeling by Estimating the Ratios of the Data Distribution},
  author    = {Lou, Aaron and Meng, Chenlin and Ermon, Stefano},
  booktitle = {Proceedings of the 41st International Conference on Machine Learning},
  series    = {Proceedings of Machine Learning Research},
  volume    = {235},
  pages     = {32819--32848},
  year      = {2024},
  publisher = {PMLR}
}

@inproceedings{
latentforcing,
title={Latent Forcing: Reordering the Diffusion Trajectory for Pixel-Space Image Generation},
author={Alan Baade and Eric Ryan Chan and Kyle Sargent and Changan Chen and Justin Johnson and Ehsan Adeli and Li Fei-Fei},
booktitle={Forty-third International Conference on Machine Learning},
year={2026},
url={https://openreview.net/forum?id=LFUaHUCO1a}
}

@misc{cobit,
      title={CoBit: Language Modeling with Bitstream Diffusion}, 
      author={Georgios Batzolis and Mark Girolami and Luca Ambrogioni},
      year={2026},
      eprint={2605.07013},
      archivePrefix={arXiv},
      primaryClass={cs.CL},
      url={https://arxiv.org/abs/2605.07013}, 
}

@misc{shabalin2026gaussiandiffusionmodelsfail,
      title={Why Gaussian Diffusion Models Fail on Discrete Data and How to Prevent It?}, 
      author={Alexander Shabalin and Simon Elistratov and Viacheslav Meshchaninov and Ildus Sadrtdinov and Dmitry Vetrov},
      year={2026},
      eprint={2604.02028},
      archivePrefix={arXiv},
      primaryClass={cs.CL},
      url={https://arxiv.org/abs/2604.02028}, 
}

@misc{deschenaux2026languagemodelinghypersphericalflows,
      title={Language Modeling with Hyperspherical Flows}, 
      author={Justin Deschenaux and Caglar Gulcehre},
      year={2026},
      eprint={2605.11125},
      archivePrefix={arXiv},
      primaryClass={cs.LG},
      url={https://arxiv.org/abs/2605.11125}, 
}

@inproceedings{
kouzelis2025boosting,
title={Boosting Generative Image Modeling via Joint Image-Feature Synthesis},
author={Theodoros Kouzelis and Efstathios Karypidis and Ioannis Kakogeorgiou and Spyros Gidaris and Nikos Komodakis},
booktitle={The Thirty-ninth Annual Conference on Neural Information Processing Systems},
year={2025},
url={https://openreview.net/forum?id=i4qAfV04rZ}
}

@inproceedings{
zhou2025next,
title={Next Semantic Scale Prediction via Hierarchical Diffusion Language Models},
author={Cai Zhou and Chenyu Wang and Dinghuai Zhang and Shangyuan Tong and Yifei Wang and Stephen Bates and Tommi Jaakkola},
booktitle={The Thirty-ninth Annual Conference on Neural Information Processing Systems},
year={2025},
url={https://openreview.net/forum?id=kUzBGEuu7w}
}

@misc{wu2026rethinkingtokenpredictiontreestructured,
      title={Rethinking Token Prediction: Tree-Structured Diffusion Language Model}, 
      author={Zihao Wu and Haoming Yang and Juncheng Dong and Vahid Tarokh},
      year={2026},
      eprint={2604.03537},
      archivePrefix={arXiv},
      primaryClass={cs.CL},
      url={https://arxiv.org/abs/2604.03537}, 
}

@inproceedings{
franca2026hackinggenerativeperplexityunconditional,
title={Hacking Generative Perplexity: Why Unconditional Text Evaluation Needs Distributional Metrics},
author={Antonio Franca and Alexander Tong},
booktitle={ICML 2026 Workshop on Structured Probabilistic Inference {\&} Generative Modeling},
year={2026},
url={https://openreview.net/forum?id=f1gwcTDwh9}
}

@misc{khelifa2026diffusionmodelsobservegradients,
      title={Diffusion Models Observe Only Gradients: A Geometric Perspective on Score Matching Errors}, 
      author={Naïl B. Khelifa and Richard E. Turner and Ramji Venkataramanan},
      year={2026},
      eprint={2606.06179},
      archivePrefix={arXiv},
      primaryClass={stat.ML},
      url={https://arxiv.org/abs/2606.06179}, 
}

@article{
oquab2024dinov,
title={{DINO}v2: Learning Robust Visual Features without Supervision},
author={Maxime Oquab and Timoth{\'e}e Darcet and Th{\'e}o Moutakanni and Huy V. Vo and Marc Szafraniec and Vasil Khalidov and Pierre Fernandez and Daniel Haziza and Francisco Massa and Alaaeldin El-Nouby and Mido Assran and Nicolas Ballas and Wojciech Galuba and Russell Howes and Po-Yao Huang and Shang-Wen Li and Ishan Misra and Michael Rabbat and Vasu Sharma and Gabriel Synnaeve and Hu Xu and Herve Jegou and Julien Mairal and Patrick Labatut and Armand Joulin and Piotr Bojanowski},
journal={Transactions on Machine Learning Research},
issn={2835-8856},
year={2024},
url={https://openreview.net/forum?id=a68SUt6zFt},
note={Featured Certification}
}

@misc{owt,  
	title={OpenWebText Corpus},
	author={Aaron Gokaslan and Vanya Cohen},
	howpublished={\url{http://Skylion007.github.io/OpenWebTextCorpus}}, 
	year={2019}
}

@inproceedings{lm1b,
  title={One billion word benchmark for measuring progress in statistical language modeling},
  author={Ciprian Chelba and Tomas Mikolov and Mike Schuster and Qi Ge and Thorsten Brants and Phillipp Koehn and Tony Robinson},
  booktitle={Interspeech},
  year={2013},
  url={https://api.semanticscholar.org/CorpusID:14136307}
}

@inproceedings{
liu2023tinygsm,
title={Tiny{GSM}: achieving 80\% on {GSM}8k with one billion parameters},
author={Bingbin Liu and Sebastien Bubeck and Ronen Eldan and Janardhan Kulkarni and Yuanzhi Li and Anh Nguyen and Rachel Ward and Yi Zhang},
booktitle={The 3rd Workshop on Mathematical Reasoning and AI at NeurIPS'23},
year={2023},
url={https://openreview.net/forum?id=ROOVUBZp8v}
}

@misc{gsm8k,
      title={Training Verifiers to Solve Math Word Problems}, 
      author={Karl Cobbe and Vineet Kosaraju and Mohammad Bavarian and Mark Chen and Heewoo Jun and Lukasz Kaiser and Matthias Plappert and Jerry Tworek and Jacob Hilton and Reiichiro Nakano and Christopher Hesse and John Schulman},
      year={2021},
      eprint={2110.14168},
      archivePrefix={arXiv},
      primaryClass={cs.LG},
      url={https://arxiv.org/abs/2110.14168}, 
}

@misc{touvron2023llama2openfoundation,
      title={Llama 2: Open Foundation and Fine-Tuned Chat Models}, 
      author={Hugo Touvron and Louis Martin and Kevin Stone and Peter Albert and Amjad Almahairi and Yasmine Babaei and Nikolay Bashlykov and Soumya Batra and Prajjwal Bhargava and Shruti Bhosale and Dan Bikel and Lukas Blecher and Cristian Canton Ferrer and Moya Chen and Guillem Cucurull and David Esiobu and Jude Fernandes and Jeremy Fu and Wenyin Fu and Brian Fuller and Cynthia Gao and Vedanuj Goswami and Naman Goyal and Anthony Hartshorn and Saghar Hosseini and Rui Hou and Hakan Inan and Marcin Kardas and Viktor Kerkez and Madian Khabsa and Isabel Kloumann and Artem Korenev and Punit Singh Koura and Marie-Anne Lachaux and Thibaut Lavril and Jenya Lee and Diana Liskovich and Yinghai Lu and Yuning Mao and Xavier Martinet and Todor Mihaylov and Pushkar Mishra and Igor Molybog and Yixin Nie and Andrew Poulton and Jeremy Reizenstein and Rashi Rungta and Kalyan Saladi and Alan Schelten and Ruan Silva and Eric Michael Smith and Ranjan Subramanian and Xiaoqing Ellen Tan and Binh Tang and Ross Taylor and Adina Williams and Jian Xiang Kuan and Puxin Xu and Zheng Yan and Iliyan Zarov and Yuchen Zhang and Angela Fan and Melanie Kambadur and Sharan Narang and Aurelien Rodriguez and Robert Stojnic and Sergey Edunov and Thomas Scialom},
      year={2023},
      eprint={2307.09288},
      archivePrefix={arXiv},
      primaryClass={cs.CL},
      url={https://arxiv.org/abs/2307.09288}, 
}

@inproceedings{onefitsall,
author = {Bao, Fan and Nie, Shen and Xue, Kaiwen and Li, Chongxuan and Pu, Shi and Wang, Yaole and Yue, Gang and Cao, Yue and Su, Hang and Zhu, Jun},
title = {One transformer fits all distributions in multi-modal diffusion at scale},
year = {2023},
publisher = {JMLR.org},
booktitle = {Proceedings of the 40th International Conference on Machine Learning},
articleno = {72},
numpages = {26},
location = {Honolulu, Hawaii, USA},
series = {ICML'23}
}

@misc{pynadath2026generativefrontiersevaluationmatters,
      title={Generative Frontiers: Why Evaluation Matters for Diffusion Language Models}, 
      author={Patrick Pynadath and Jiaxin Shi and Ruqi Zhang},
      year={2026},
      eprint={2604.02718},
      archivePrefix={arXiv},
      primaryClass={cs.LG},
      url={https://arxiv.org/abs/2604.02718}, 
}

@inproceedings{continuously-augmented,
title = {Continuously Augmented Discrete Diffusion model for Categorical Generative Modeling},
booktitle = {The Fourteenth International Conference on Learning Representations},
author = {Huangjie Zheng and Shansan Gong and Ruixiang Zhang and Tianrong Chen and Jiatao Gu and Mingyuan Zhou and Navdeep Jaitly and Yizhe Zhang},
year = {2026},
URL = {https://openreview.net/forum?id=JNAZ3e7Bwt}
}

@misc{zhou2026coevolutionarycontinuousdiscretediffusion,
      title={Coevolutionary Continuous Discrete Diffusion: Make Your Diffusion Language Model a Latent Reasoner}, 
      author={Cai Zhou and Chenxiao Yang and Yi Hu and Chenyu Wang and Chubin Zhang and Muhan Zhang and Lester Mackey and Tommi Jaakkola and Stephen Bates and Dinghuai Zhang},
      year={2026},
      eprint={2510.03206},
      archivePrefix={arXiv},
      primaryClass={cs.AI},
      url={https://arxiv.org/abs/2510.03206}, 
}
\bibliographystyle{iclr2027_conference}

\newpage

\appendix

\section*{Appendix}

\section{Semantic Clustering of Token Embeddings}
\label{app:clustering}

This appendix details the construction of the surjective map $\Gamma : \mathcal{V} \to \mathcal{C}$
used in Section \ref{subsec:hiearchical_construc}. $\Gamma$ is a spherical $K$-means partition of a
pretrained token-embedding matrix, computed offline and frozen before training, followed by a
frequency-driven relabelling of the cluster \emph{codes} (Algorithm \ref{alg:clustering}).

\subsection{Requirements on \texorpdfstring{$\Gamma$}{Gamma}}

\para{Determinism.} Proposition \ref{prop:deterministic_latent} requires $y^2 = \Gamma(y^1)$ to be a
deterministic function of the tokens, for our model of the joint distribution to be a valid model of the token distribution. $\Gamma$ is therefore computed once, offline, and frozen: it is a
property of the tokenizer and of the pretrained embeddings, not a learned component. It is independent of the underlying model's token representation. This is what allows discarding the cluster modality at the end of sampling.

\para{Balance.} The per-position entropy of the high-level modality (clusters) is $H(y^2)$, maximal and equal to
$\log n$ for equiprobable clusters. Balanced clusters are more expressive, offering a better anchor to the diffusion process.

\para{Semantics.} Balance alone is not sufficient: Section \ref{subsec:ablations} shows that a
uniformly random cluster map, balanced by construction, recovers none of the gain. What the model
exploits is that tokens sharing a cluster are close in embedding space, so resolving the cluster genuinely narrows the token posterior and provides a strong anchor to the denoising process.

\subsection{Input embeddings}
\label{app:clust:embeddings}

We cluster the input embedding matrix $E \in \mathbb{R}^{|\mathcal{V}| \times D}$ of a pretrained
LangFlow checkpoint \citep{chen2026langflow}, $D = 768$ (\texttt{langflow-lm1b} for LM1B,
\texttt{langflow-owt} for OpenWebText): a flow LM using the same tokenizers. LangFlow normalizes each embedding row to the sphere of radius
$\sqrt{D}$, so that all information is carried by directions. We accordingly unit-normalize
the rows of $E$ and cluster on the unit sphere.\footnote{We read the raw embedding parameter, which
LangFlow row-normalizes at inference time.}

\para{TinyGSM.} On TinyGSM, following \cite{deschenaux2026languagemodelinghypersphericalflows, cobit}, we use the SmolLM-135M tokenizer plus an added \texttt{[PAD]}
($|\mathcal{V}| = 49{,}153$). No LangFlow checkpoint shares that tokenizer. We use the input embeddings of an S-FLM \citep{deschenaux2026languagemodelinghypersphericalflows}
checkpoint, which also constrains embeddings to a sphere.

\para{H-CoBit on OpenWebText: span averaging.} On OpenWebText, CoBit does not operate directly on
GPT-2 tokens but on \texttt{gpt2id\_bpe16} \emph{code} tokens, a second-stage BPE tokenizer over GPT-2 ids, yielding
$|\mathcal{V}| = 65{,}536$ -- which no pretrained model embeds. We therefore take each code token's
vector to be the mean of the LangFlow-OWT embeddings of the GPT-2 span it expands to. Four code tokens (\texttt{[PAD]}, \texttt{[UNK]}, \texttt{[EOSEQ]}, \texttt{'\textbackslash n'}) expand to no
span and receive no vector. The same happens for SmolLM-135M's \texttt{[PAD]} token (for TinyGSM). They are attributed to a separate cluster, of id
$n-1$, and $K$-means is run with $n-1$ centroids on the remaining tokens.

\subsection{Clustering}
\label{app:clust:kmeans}

We run Lloyd's algorithm with $K$-means++ initialization (\texttt{sklearn.cluster.KMeans}) on the
unit-normalized token embeddings. We keep the best of $n_{\mathrm{init}} = 10$ restarts, each capped at
$T = 300$ iterations and stopped early below a relative centroid shift of $\varepsilon = 10^{-4}$.

\subsection{Frequency-driven code rebalancing for H-CoBit}
\label{app:clust:rebalance}

$K$-means numbers its clusters arbitrarily and that number is written directly into the cluster bits, implying potentially imbalanced per-bit marginals
$p_b = \mathbb{P}(\text{bit } b = 1)$ for H-CoBit. We therefore search for a permutation $\tau$ of the cluster codes
maximizing the frequency-weighted $\sum_b H(p_b)$, by 2-opt local search over code pairs. Frequencies
are estimated over a prefix of the on-disk token cache ($50$M / $8$M tokens for OWT / TinyGSM). For OWT $n = 16$, the realized per-bit
entropy goes from $3.75$ to $3.99$ of $4.0$ ($35$ swaps). We only apply code rebalancing on OWT and TinyGSM.

\begin{algorithm}[h]
\caption{Construction of $\Gamma$ (offline, once per (tokenizer, $n$) pair)}
\label{alg:clustering}
\begin{algorithmic}[1]
\Require pretrained embeddings $E$, clustered vocabulary $\mathcal{V}$, clusters $n$, corpus cache
\Ensure frozen map $\Gamma : \mathcal{V} \to \mathcal{C}$, $|\mathcal{C}| = n$
\State $e_v \gets E_v$, or the mean of $E$ over $v$'s subword span for code-token vocabularies (H-CoBit on OWT)
\State $\mathcal{F} \gets \{ v : e_v \text{ exists} \}$; \quad
  $\hat{e}_v \gets e_v / \lVert e_v \rVert$ for $v \in \mathcal{F}$
\If{reserving a structural cluster} \Comment{OWT, TinyGSM}
\State $\Gamma(v) \gets n-1$ for all $v \notin \mathcal{F}$; \quad $K \gets n-1$
\Else \State $K \gets n$ \Comment{LM1B: $\mathcal{F} = \mathcal{V}$}
\EndIf
\State $\Gamma|_{\mathcal{F}} \gets \textsc{KMeans}\big(\{\hat{e}_v\}_{v \in \mathcal{F}},\, K;\;\text{K-means++},\big) $
  
\State $f_c \gets$ corpus frequency of cluster $c$, over a prefix of the token cache
\State $\tau \gets \arg\max_{\tau \in \mathfrak{S}_n} \sum_{b} H\big(p_b(\tau; f)\big)$
  by 2-opt local search \Comment{codes only}
\State \Return $\tau \circ \Gamma$
\end{algorithmic}
\end{algorithm}

\subsection{Hyperparameters and measured quality}

Table \ref{tab:clustering_hparams} lists the settings used for all runs. The clustering is cheap: a few minutes per map on CPU.

\begin{table}[h]
\centering
\small
\caption{\textbf{Clustering hyperparameters.} $\Gamma$ is computed once per (tokenizer, $n$) pair and
frozen before training.}
\label{tab:clustering_hparams}
\begin{tabular}{@{}ll@{}}
\toprule
Parameter & Value \\
\midrule
Embeddings & pretrained LangFlow/S-FLM input embeddings, $D = 768$ \\
Metric & Euclidean on unit-normalized rows ($\equiv$ cosine) \\
Clusters $n$ & LM1B $\{8, \ldots, 256\}$, OWT / TinyGSM $\{16, 32, 64\}$; $n = 16$ default \\
Initialization & $K$-means++, $n_{\mathrm{init}} = 10$ restarts, best inertia \\
Max iterations $T$ / tolerance $\varepsilon$ & $300$ / $10^{-4}$ (relative centroid shift) \\
Seed & $42$ \\
Structural cluster & reserved on OWT, TinyGSM (one of the $n$ ids); not LM1B \\
\bottomrule
\end{tabular}
\end{table}

\begin{table}[h]
\centering
\small
\caption{\textbf{Measured clustering quality at $n = 16$.} $n_{\mathrm{sem}}$ excludes the reserved
structural cluster. Silhouette / DBI / CH are computed on the fitted
tokens only, silhouette on a $10{,}000$-token subsample. $H_{\mathrm{type}}$ is the type-occupancy
entropy normalized by $\log_2 n_{\mathrm{sem}}$.}
\label{tab:clustering_quality}
\begin{tabular}{@{}lccccccc@{}}
\toprule
Corpus & $|\mathcal{V}|$ & $n_{\mathrm{sem}}$ & size min / mean / max & Sil.\ $\uparrow$ & DBI $\downarrow$ & CH $\uparrow$ & $H_{\mathrm{type}}$ \\
\midrule
LM1B        & 30{,}522 & 16 & 848 / 1908 / 4210  & 0.050 & 3.50 & 506  & 0.969 \\
OpenWebText & 65{,}536 & 15 & 2279 / 4369 / 7671 & 0.057 & 3.41 & 1300 & 0.978 \\
\bottomrule
\end{tabular}
\end{table}

Silhouette values are small in absolute terms ($0.003$--$0.075$ across our maps), as expected for
high-dimensional embedding clustering at these $n$. We read them comparatively between embedding
sources at fixed $n$, not as evidence of well-separated clusters. Qualitatively the partitions are
interpretable. On LM1B, $n = 16$ separates punctuation, subword continuations, function words and
concrete nouns (\emph{head, eyes, house, hand, room, water, door, river}). Its largest cluster is the
block of BERT's reserved \texttt{[unused$k$]} ids, which never occur in the corpus.
Table \ref{tab:owt_clusters} gives the OWT case.

\begin{table}[h]
\centering
\small
\caption{\textbf{OpenWebText clusters at $n = 16$.} The seven most-used semantic clusters, with their
share of a $35.5$M-token sample and their most frequent \texttt{gpt2id\_bpe16} code tokens (which may
span several GPT-2 tokens). Ids are the $K$-means labels, prior to the code rebalancing of
Section \ref{app:clust:rebalance}, which permutes them without changing membership.}
\label{tab:owt_clusters}
\begin{tabular}{@{}clp{7.4cm}@{}}
\toprule
Id & Usage & Most frequent members \\
\midrule
11 & 16.5\% & \texttt{the}, \texttt{of}, \texttt{in}, \texttt{a}, \texttt{for}, \texttt{of the}, \texttt{in the}, \texttt{with}, \texttt{by}, \texttt{on} \\
12 & 13.0\% & \texttt{to}, \texttt{that}, \texttt{is}, \texttt{'s}, \texttt{are}, \texttt{as}, \texttt{it}, \texttt{were}, \texttt{has}, \texttt{have} \\
9  &  7.9\% & \texttt{,}, \texttt{and}, \texttt{, the}, \texttt{, and}, \texttt{or}, \texttt{, a}, \texttt{, but}, \texttt{--}, \texttt{;} \\
0  &  6.6\% & \texttt{.}, \texttt{:}, \texttt{. The}, \texttt{.\textbackslash n\textbackslash nThe}, \texttt{.''}, \texttt{. He}, \texttt{?} \\
2  &  6.2\% & \texttt{was}, \texttt{had}, \texttt{said}, \texttt{ed}, \texttt{has been}, \texttt{made}, \texttt{did}, \texttt{left}, \texttt{says}, \texttt{called} \\
4  &  4.7\% & \texttt{two}, \texttt{2}, \texttt{1}, \texttt{3}, \texttt{4}, \texttt{three}, \texttt{5}, \texttt{\$}, \texttt{10} \\
3  &  4.2\% & \texttt{National}, \texttt{State}, \texttt{House}, \texttt{New}, \texttt{News}, \texttt{City}, \texttt{Google}, \texttt{New York}, \texttt{University} \\
\bottomrule
\end{tabular}
\end{table}

\subsection{The random-cluster control}

In Section \ref{subsec:ablations} we present an ablation replacing $\Gamma$ by a map
$\Gamma_{\mathrm{rand}}$ drawn uniformly at random over $\mathcal{C}^{\mathcal{V}}$, independently of the
embeddings, then frozen exactly like $\Gamma$. It is matched to our clustering in every non-semantic
respect. The only thing it destroys is the alignment between cluster
identity and embedding geometry.

\section{Extended Background}
\label{app:extended_background}

In this section, we provide further details about the general background set up for continuous diffusion language models in Section \ref{sec:background}, and provide the link between our general setup and existing Continuous DLM's instantiations from the literature.

\subsection{Auto-Regressive \emph{vs} Diffusion Language Models}

Auto-Regressive models conditionally factorize the joint data distribution left-to-right: $p_{\theta}(y):=\prod_{l=1}^{L}p_{\theta}(y^l|y^{< l})$. 
This allows computing the exact likelihood of $p_{\theta}$ and mimics human speech, at the cost of slow iterative sampling. Diffusion models instead generate all positions in parallel through iterative refinement, with the joint distribution defined by the probability-flow Ordinary Differential Equation (ODE) (\eqref{equ:pfode}), allowing for fast parallel sampling. Here, sampling costs are governed by the sampling steps rather than only the sequence length.

\subsection{Link to score matching}

As presented in Section \ref{sec:background}, we consider a denoiser $D_{\theta}(x_{\sigma}, \sigma) \approx \mathbb{E}[x_0 | x_{\sigma}] \in E^{L \times d}$. By Tweedie's formula, it gives direct access to the score of the marginal density $p_\sigma$,
\begin{align}
\label{equ:score}
    \nabla_x \log p_{\sigma}(x)
    = \frac{D_{\theta}(x, \sigma) - x}{\sigma^{2}}.
\end{align}

\subsection{From general \texorpdfstring{$(\alpha_t, \sigma_t)$}{(alpha,sigma)} parametrizations to the variance-exploding convention}
\label{app:ve_correspondence}

Section \ref{sec:background} adopts the variance-exploding (VE) convention of \eqref{equ:forward}, also used by \cite{cobit}. Flow matching and variance-preserving formulations instead use an affine Gaussian path
\begin{align}
\label{equ:affine}
    x_t = \alpha_t x_0 + \sigma_t \epsilon,
    \qquad \epsilon \sim \mathcal{N}(0, I_{Ld}),
    \qquad \alpha_t > 0,
\end{align}
with $\alpha_t$ and $\sigma_t$ being monotone along the path. Setting
\begin{align}
\label{equ:ve_map}
    \sigma \;=\; \frac{\sigma_t}{\alpha_t},
    \qquad
    x_{\sigma} \;=\; \frac{x_t}{\alpha_t}
    \qquad \text{gives} \qquad
    x_{\sigma} = x_0 + \sigma \epsilon,
\end{align}
which is \eqref{equ:forward}.

At fixed $t$ the map $x_t \mapsto x_{\sigma}$ is a bijection, so conditioning on either variable gives the same posterior, $\mathbb{E}[x_0 \mid x_t] = \mathbb{E}[x_0 \mid x_{\sigma}] = D_{\theta}(x_{\sigma}, \sigma)$: any network trained to predict clean data under \eqref{equ:affine} is a denoiser, up to input rescaling. 

Furthermore, $\sigma$ is a monotone function of $t$. A monotone change of time variable leaves the continuous ODE trajectory unchanged. It affects only the density of training noise levels which is what the schedules of Section \ref{subsec:sampling} control. 

\subsection{From our general setup to the continuous DLMs literature}
\label{app:model_instantiation}

Tables \ref{tab:app_models_repr} and \ref{tab:app_models_train} instantiate, respectively, the representation and corruption process of Section \ref{sec:background}, and its objective, schedule and sampler, for several existing continuous DLMs.

\begin{table}[h]
\centering
\footnotesize
\setlength{\tabcolsep}{2.5pt}
\renewcommand{\arraystretch}{1.15}
\caption{\textbf{Representation and corruption process.} Instantiation of the encoder $\mathcal{E}$, of the forward process, and of the decoder $\mathcal{D}$ of Section \ref{sec:background}. The VE-equivalent noise level follows from the rescaling \eqref{equ:ve_map}.}
\label{tab:app_models_repr}
\begin{tabular}{@{}>{\raggedright\arraybackslash}p{1.5cm} >{\raggedright\arraybackslash}p{3.0cm} >{\raggedright\arraybackslash}p{3.3cm} >{\raggedright\arraybackslash}p{2.0cm} >{\raggedright\arraybackslash}p{2.7cm}@{}}
\toprule
Model & Representation $x_0 = \mathcal{E}(y)$ & Corruption path & VE-equivalent $\sigma$ & Decoder $\mathcal{D}$ \\
\midrule
Analog Bits
& $m = \lceil \log_2 |\mathcal{V}| \rceil$ analog bits in $\{\pm 1\}$; $d = m$
& VP (cosine), $x_t = \sqrt{\bar\alpha_t}\, x_0 + \sqrt{1-\bar\alpha_t}\, \epsilon$
& $\sqrt{(1-\bar\alpha_t)/\bar\alpha_t}$
& threshold bits at $0$ \\
CoBit
& analog bits in $\{0,1\}$, $m=15$ (LM1B) or $16$ (OWT codec); $d=m$, patched to length $L$
& VE, native: $x_\sigma = x_0 + \sigma\epsilon$, $\sigma_{\mathrm{data}} = 1/2$
& $\sigma$ (native)
& threshold bit probabilities, invert codec \\
LangFlow
& learned embeddings, $\ell_2$-normalized and scaled by $\sqrt{D}$; $d = D = 768$
& VP $\gamma$-path, $z_l = \alpha_l z + \sigma_l \epsilon$, $\sigma_l^2 = \mathrm{sigmoid}(l)$
& $e^{l/2}$, i.e. $l = 2 \log \sigma$
& $\arg\max$ of predicted token probabilities \\
FLM,\newline FMLM
& one-hot vectors; $d = |\mathcal{V}|$
& linear interpolant $I_t = (1-t)\,\epsilon + t\, x_1$
& $(1-t)/t$
& $\arg\max$ per position \\
ELF
& frozen T5-small \emph{contextual} embeddings ($512$), linear bottleneck to $128$
& rectified flow, $z_t = t\, x + (1-t)\, \epsilon$
& $(1-t)/t$
& network in \emph{decode} mode at $t=1$, then unembedding \\
S-FLM
& learned unit-norm embeddings, $E = \mathbb{S}$; $d = 768$
& geodesic, $z_t = \mathrm{SLERP}(z_0, z_1, \alpha_t)$, $z_0 \sim \mathcal{U}(\mathbb{S}^{d-1})$
& --- (not additive Gaussian)
& $\arg\max$ of $p^{\theta}_{1|t}$ \\
\bottomrule
\end{tabular}
\end{table}

\begin{table}[h]
\centering
\footnotesize
\setlength{\tabcolsep}{2.5pt}
\renewcommand{\arraystretch}{1.15}
\caption{\textbf{Prediction target, objective, schedule and sampler of continuous DLMs.} All models estimate the same posterior, but differ in how it is parametrized, how noise levels are allocated, and how much stochasticity the sampler injects.}
\label{tab:app_models_train}
\begin{tabular}{@{}>{\raggedright\arraybackslash}p{1.3cm} >{\raggedright\arraybackslash}p{2.9cm} >{\raggedright\arraybackslash}p{2.2cm} >{\raggedright\arraybackslash}p{3.0cm} >{\raggedright\arraybackslash}p{3.2cm}@{}}
\toprule
Model & Network output & Objective & Noise / time schedule & Sampler, stochasticity \\
\midrule
Analog Bits 
\citeyearpar{chen2022analogbits}
& $\hat{x}_0$ (analog bits)
& MSE, $w \equiv 1$
& cosine (VP)
& DDIM / ancestral; SC ($p=0.5$); asymmetric time intervals \\
CoBit \citeyearpar{cobit}
& bit posteriors in $(0,1)^{Lm}$ via matched-filter residual $\ell_\theta = r_\theta + \mathrm{clip}\big((x_\sigma - \tfrac12)/\sigma^2\big)$
& weighted MSE, EDM $w(\sigma) = \frac{\sigma^2 + \sigma_{\mathrm{data}}^2}{\sigma^2 \sigma_{\mathrm{data}}^2}$
& online entropy-rate, $\pi_\alpha(u) \propto g(\sigma)\, \hat{h}_{\log}(\sigma)^{\alpha}$ with $\alpha = \tfrac12$ and $\hat{h}_{\log} \approx e(\sigma)/\sigma^2$
& DDIM PF-ODE $+$ full-band EDM churn $\gamma_i = S_{\mathrm{churn}}/N$, entropy-gated $\lambda_{\mathrm{ent}} \approx S_{\mathrm{churn}} \pi_\alpha$; SC ($p=0.5$, carry mode) \\
LangFlow \citeyearpar{chen2026langflow}
& token probabilities $\hat{x}_\theta \in \Delta^{|\mathcal{V}|-1}$; denoiser $\hat{z}_\theta = E^\top \hat{x}_\theta$
& cross-entropy (Bregman flow matching) $+$ scheduler loss
& learnable Gumbel density over $l$ fitted to the information gain $H'_\gamma\approx |\mathrm{d}\hat{\mathcal{L}}/\mathrm{d}t|$
& Euler on the $l$-path, deterministic by design; SC ($p=0.25$) \\
FLM,\newline FMLM \citeyearpar{lee2026flm}
& posterior $p_{1|t}$ through a tokenwise softmax ($= \mathbb{E}[x_1 \mid I_t]$)
& cross-entropy; FMLM: KL semigroup loss on the two-time denoiser $\delta_{s,t}$
& decoding-error reparametrization $\tau(t) = 1 - \tfrac{|\mathcal{V}|}{|\mathcal{V}|-1} P_e(t)$
& Euler ODE; FMLM performs one/few-step flow-map jumps; no SC. (We add stochastic sampling). \\
ELF \citeyearpar{hu2026elf}
& clean embeddings ($x$-prediction), shared-weight decoding head at $t=1$
& velocity MSE ($w \propto (1-t)^{-2}$, $80\%$) $+$ token cross-entropy at $t=1$ ($20\%$)
& logit-normal over $t$ ($P_{\mathrm{mean}} = -1.5$, $P_{\mathrm{std}} = 0.8$), at training and inference
& Euler ODE, or noise re-injection (``SDE'', scale $\gamma$); SC ($p=0.5$) $+$ training-time CFG \\
S-FLM \citeyearpar{lee2026flm}
& posterior $p^{\theta}_{1|t}$; tangent average $\bar{u} = \sum_v p^{\theta}_{1|t}(v)\, \log_{z_t}(\hat{e}_v)$
& cross-entropy
& truncation to the high-noise band $\alpha \le \alpha^{\star}(\delta)$, plus adaptive refit from $|\mathrm{d}\hat{\mathcal{L}}/\mathrm{d}t|$ (information gain).
& geodesic Euler; exact, one-sample stochastic or top-$k$ velocity; temperature $T$; no SC \\
\bottomrule
\end{tabular}
\end{table}

\para{Analog Bits \citep{chen2022analogbits}.} 
It is not a language model, but it is a key step towards CoBit's framework \citep{cobit} on which we rely on extensively in this work.

\para{CoBit \citep{cobit}.} 
The closest model to our setup. Its denoiser outputs bitwise clean probabilities. Bits being binary, output probabilities equal $\mathbb{E}[x_0 | x_\sigma]$, so $D_{\theta}(x_{\sigma},\sigma)=\mathbb{E}[x_0 | x_\sigma]$ holds exactly and the score follows from \eqref{equ:score}. It further utilizes a matched residual parametrization of the bitwise probabilities.

\para{LangFlow \citep{chen2026langflow}.} LangFlow uses a Flow parametrization and represents tokens via learned embeddings. Under \eqref{equ:ve_map}, its time variable $l$ is exactly $l = 2 \log \sigma$. The network predicts token probabilities and is trained with cross-entropy; the denoiser from Section \ref{sec:background} is recovered as the embedding expectation $\hat{z}_\theta = E^\top \hat{x}_\theta$. Sampling is restricted to the deterministic ODE to allow flow-map distillation.

\para{FLM / FMLM \citep{lee2026flm}.} FLM represents tokens as one-hot vectors. Hence, the posterior mean is the categorical posterior, which allows a tokenwise softmax denoiser trained with cross-entropy. FLM relies on a decoding error rate $P_e(t)$ based time reparametrization. This is very close conceptually to our information-uniform schedule in Section \ref{tab:dual_schedul}, as we approximate conditional entropy via validation loss.

\para{ELF \citep{hu2026elf}.} ELF is the only model that is not applying token-level supervision along the trajectory: discretization is confined to $t=1$, where the same network, in decoding mode, is followed by an unembedding layer. This keeps ELF close to the image-domain instantiation, allowing direct use of existing techniques from that domain. ELF also keeps an EDM-style logit-normal noise density rather than fitting an information-uniform schedule. 
In general, ELF is the model that departs most from our framework, and from existing continuous DLMs.

\para{S-FLM \citep{deschenaux2026languagemodelinghypersphericalflows}.} S-FLM embeds tokens on the hypersphere. Noise is injected by rotation rather than addition, along the geodesic from the clean embedding to a uniform sample on $\mathbb{S}^{d-1}$. The marginal velocity averages the tangent directions $\log_{z_t}(\hat{e}_v)$ under the posterior. Tokens are then decoded via a greedy approach.

\para{Limitations of the framework.} Our general setting (Section \ref{sec:background}) applies seamlessly to most existing continuous DLMs, with three caveats which do not affect the construction of Section \ref{sec:method}: \emph{(i)} ELF's encoder is contextual, so that $x_0 = \mathcal{E}(y)$ is not obtained position-wise and $\mathcal{D} \circ \mathcal{E} = \mathrm{id}$ only up to decoder error. \emph{(ii)} S-FLM's forward process is not a pure rescaling of \eqref{equ:forward}, since it swaps additive for geodesic interpolation. The per-modality construction of Section \ref{subsec:dual_noise} however transfers unchanged. \emph{(iii)} LangFlow, FLM and S-FLM train with cross-entropy rather than the regression loss of \eqref{equ:mse}, but still estimate the same object.

\subsection{Information-uniform noise schedules}
\label{app:info_uniform}

The information-guided noise allocation approach, originally proposed by \cite{stancevic2025entropic} and \cite{raya2026infonoise}, is adopted by the recent continuous diffusion for text works in varying forms. 

We adopt the following general approach following \cite{chen2026langflow,cobit}. Let $\pi$ denote the density of sampling steps over $u=\log~\sigma$, and let $\hat{h}_{log}(u)$ be an estimator of the conditional entropy rate per unit of log noise,
\begin{align}
    h_{\log}(\sigma)
    = \frac{d}{d \log \sigma}\, H\!\left(y \mid x_{\sigma}\right).
\end{align}
We aim to set $\pi$ so that noise allocation matches information gain: $\pi(u) \propto \hat{h}_{log}(u)$ \footnote{CoBit \citep{cobit} more specifically sets  $\pi(u) \propto \frac{\sigma^3}{\sigma^3 + 0.1^3}\hat{h}_{log}(u)^{\frac1{2}}$}, and place the $N$ solver points uniformly in the cumulative distribution function (CDF) of $\pi$,
$\sigma_i = exp(F^{-1}\!\left(1 - i/N\right))$ with $F$ the CDF of $\pi$ over
$u$. This yields a grid on which each step reveals approximately equal
conditional information about the tokens. In practice, $h_{log}(u)$ is approximated via validation loss after a loglinear schedule warmup phase. $\pi$ is either constrained to a known parametric family such as Gumbel densities \citep{chen2026langflow} or fit as an unnormalized density \citep{cobit}. We note that \cite{lee2026flm} instead define $\hat{h}_{log}(u)$ as the decoding error of their diffused one hot token vectors, which results in a similar shaped reparametrization although its mechanism differs. \cite{deschenaux2026languagemodelinghypersphericalflows} fit the schedule to the derivative of the loss profile and additionally truncate its support, while among the models we cover only \cite{hu2026elf} keeps a fixed, image-style noise density.

\subsection{Self conditioning}

Self-conditioning (SC), introduced by \cite{chen2022analogbits}, feeds the denoiser its own previous clean-data estimate as an auxiliary input, so that consecutive sampling steps are no longer independent forward passes. \cite{shabalin2026gaussiandiffusionmodelsfail} argue that this is what prevents the trajectory from drifting away from the discrete modes of the data, which explains why SC is markedly more useful for text than for images.

 At training time, this entails running a no-gradient forward pass of the model with probability $p_{sc}$ whose output clean-data estimate $x_{sc}$ is fed to the model in the subsequent steps. When no self conditioning is enabled, the model receives zeroed-out data estimates. In practice $p_{sc}=0.5$ for \cite{chen2022analogbits, cobit, hu2026elf} and $p_{sc}=0.25$ for \cite{chen2026langflow}.

 At sampling time, this entails one first step without SC, after which each previous step data estimates are fed back to the model in the next step, thus entailing no additional sampling steps.

\subsection{Instantiating the hierarchical framework}
\label{app:hierarchical_instantiation}

In Section \ref{sec:method}, we describe H-CDLM in a general setting and apply it to CoBit and FLM for experiments in Section \ref{sec:experiments}. H-CDLM can also be extended to other continuous diffusion/flow formulations; we leave the empirical evaluation of these variants to future work.

Applying H-CDLM to a given model requires three decisions only (for $k=2)$: (i) the cluster encoder $\mathcal{E}_2$, which we take to be the model's own token encoder applied to the cluster index (Section \ref{subsec:hiearchical_construc}); (ii) the loss weight $\lambda_2$ of \eqref{equ:hierarchical_objective}; (iii) and the per-modality instantiation of the schedule and sampler of Section \ref{subsec:dual_noise}. Table \ref{tab:app_models_hier} shows these for each model. No component of the base model changes: the cluster channel is concatenated to the state at every position, so architecture, objective family and sampler are inherited unchanged.

\begin{table}[h]
\centering
\footnotesize
\setlength{\tabcolsep}{2.5pt}
\renewcommand{\arraystretch}{1.15}
\caption{\textbf{Hierarchical instantiation of existing continuous DLMs.} How the $k=2$ construction of Section \ref{sec:method} applies to each model, with $n = |\mathcal{C}|$ clusters and $d_2$ the cluster embedding dimension. Added state is reported for $n=16$, the value we find optimal (Figure \ref{fig:k_lmabda_ablation}). The CoBit row corresponds to H-CoBit (Section \ref{sec:experiments}).}
\label{tab:app_models_hier}
\begin{tabular}{@{}>{\raggedright\arraybackslash}p{1.6cm} >{\raggedright\arraybackslash}p{3.4cm} >{\raggedright\arraybackslash}p{2.4cm} >{\raggedright\arraybackslash}p{5.5cm}@{}}
\toprule
Model & Cluster encoder $\mathcal{E}_2$ & Added state per position, $n=16$ & Per-modality schedule and remarks \\
\midrule
CoBit
& $\lceil \log_2 n \rceil$ analog bits appended to the token's bit patch
& $+4$ bits on $m=15/16$ (approx. $+1\%$ parameters, Section \ref{subsec:exp_setup})
& Own entropy-rate density $\pi^2_\alpha$ fitted on the cluster loss; churn can be set per-modality. Encoding the cluster \emph{inside} the token code instead of appending it underperforms. \\
LangFlow
& learned cluster embedding table $E^{(2)} \in \mathbb{R}^{n \times d_2}$, concatenated per position
& $+d_2$ channels on $D = 768$
& A second learnable Gumbel scheduler over $l^2$; the deterministic sampler leaves no churn to tune. \\
FLM,\newline FMLM
& one-hot over $\mathcal{C}$
& $+16$ on $|\mathcal{V}| = 30522$
& A second decoding-error curve $P_e^{2}(t)$, hence a second reparametrization $\tau_2$. \\
ELF
& extra embedding channels for the cluster, plus a second unembedding head at the decoding step
& $+d_2$ channels
& No information-uniform schedule to fit; the logit-normal density would have to be shifted per-modality by hand, as in image latent forcing \citep{latentforcing}. \\
S-FLM
& second unit-norm cluster table on $\mathbb{S}^{d_2-1}$, with its own SLERP interpolation
& $+d_2$ channels on $d = 768$
& Own truncation threshold $\alpha^\star_2(\delta)$, which predicts the coarse-to-fine ordering. \\
\bottomrule
\end{tabular}
\end{table}

\subsection{Evaluating diffusion language models.} Evaluating Language Models is hard, because there is no single agreed-upon measure of generation quality. Generative Perplexity (GenPPL), that is the perplexity measured by an auxiliary model (typically GPT-2 Large) on the sequences generated by the model under evaluation, has been extensively used in the continuous diffusion literature as the main evaluation metric \citep{lee2026flm, chen2026langflow, cobit}. This is preferred to Evidence Lower Bounds (ELBO) since they were shown as ill-suited to continuous diffusion \citep{khelifa2026diffusionmodelsobservegradients}, unlike for discrete DLMs for which the ELBO is a key metric. GenPPL however requires careful reporting, since it can be artificially reduced by reducing sample entropy. Even at dataset entropy, the pertinence of GenPPL remains debatable \citep{franca2026hackinggenerativeperplexityunconditional}. The standard most recent practice consists in reporting GenPPL-Entropy frontiers, which are typically increasing, with a focus on GenPPL at the dataset entropy. \cite{franca2026hackinggenerativeperplexityunconditional} also advocate for the use of distributional metrics, including MAUVE, although they come at a higher cost.
\cite{deschenaux2026languagemodelinghypersphericalflows} also introduce downstream task performance on \textit{sudoku} and GSM8k, which has the advantage of allowing fully interpretable results.

\section{Our Method's Details}
\label{app:method_details}

\subsection{The joint distribution is a valid model of the data}

Proposition~\ref{prop:deterministic_latent} is what justified discarding the
auxiliary modalities at the end of sampling. It is the text-domain counterpart
of the argument underlying \emph{latent forcing} in image diffusion
\citep{latentforcing}, where pixels are recovered from jointly diffused DINOv2
features. We
restate it in our framework and give the short proof.

\begin{proposition}
\label{prop:deterministic_latent}
Let $y^1 \sim p_{\mathrm{data}}$ and let $y^j = \Gamma_j(y^1)$ for measurable maps
$\Gamma_j$, $j = 2, \dots, k$. Write $y = (y^1, \dots, y^k)$ and let $p$ denote its
joint distribution. Then the first marginal of $p$ is $p_{\mathrm{data}}$. Consequently,
any sampler whose output is distributed according to $p$ yields an exact sampler for
$p_{\mathrm{data}}$ by returning $y^1$, irrespective of the order in which the
coordinates are resolved and of the per-modality noise schedules used to resolve them.
\end{proposition}

\begin{proof}
By construction, $y = (y^1, \Gamma_2(y^1), \dots, \Gamma_k(y^1))$ with
$y^1 \sim p_{\mathrm{data}}$, so the first coordinate of $y \sim p$ is distributed
as $p_{\mathrm{data}}$. Hence, if a sampler outputs $\hat{y} \sim p$, then
$\hat{y}^1 \sim p_{\mathrm{data}}$. This depends only on the law of $\hat{y}$, not on
the order or schedules used to produce it.
\end{proof}

\subsection{Per-modality information uniform schedulers}
\label{app:per_modality_schedul}

A key opportunity of joint diffusion is to let the denoiser exploit different noise levels per-modality. In multimodal image diffusion, this can be achieved by \emph{ordering} the modalities: the high level modality is kept cleaner than the pixels, so that resolved semantics guide pixel denoising \citep{latentforcing, kouzelis2025boosting}. In image diffusion, this requires setting an
explicit ordering of the modalities' schedules, after careful sweeps over time-shift parameters.

Our framework offers a simpler alternative: training a \emph{per-modality} information-uniform noise schedule (see Section \ref{subsec:sampling}). For each modality $j$, let $\pi^j$ denote the density of sampling steps over $u^j = \log \sigma_j$, and let $\hat{h}^j_{log}(u^j)$ be an estimator of the
conditional entropy rate per unit of log noise
\begin{align}
    h^{j}_{\log}(\sigma_j)
    = \frac{d}{d \log \sigma_j}\,
      H\!\left(y_j \,\middle|\, x_{\bm{\sigma}}\right).
\end{align}
Let $\pi^j(u^j) \propto \hat{h}^{j}_{\log}(u^j)$ with
$F_j$ its Cumulative Distribution Function (CDF), so that noise allocation matches information gain. For modality $j$, placing the $N$ solver steps $i$ uniformly in the CDF of $\pi^j$,
$\sigma_{i,j} = exp(F^{-1}_j\!\left(1 - i/N\right))$, this yields a grid on which each step reveals approximately equal
conditional information about the underlying data. We refer to Appendix \ref{app:extended_background} for details on estimating $\hat{h}^j$. Set \emph{per-modality}, this information-uniform noise schedule from Section \ref{sec:background} provides a natural candidate for the per-modality schedules. Figure \ref{fig:noise-schedules} shows the two measured densities for $k=2$. Cluster density is visibly right-shifted: cluster identity is a coarser variable, so information about it survives larger noise so that its entropy is released at higher $\sigma$. This leads to two sampling strategies:

\paragraph{Shared schedule: implicit ordering.}
If both modalities follow a single shared noise grid ($\sigma_{i,1}=\sigma_{i,2}$), the right shift implies that at any step along the sampling trajectory, a larger part of the cluster information than the token information is resolved. The coarse-to-fine ordering recovered by image diffusion through explicit time-shifts \citep{latentforcing} hence emerges \emph{implicitly}.

\paragraph{Coupled per-modality schedule: synchronizing modalities.}
Alternatively, each modality can follow its \emph{own} information-uniform grid, coupled through a single progress variable, so that every modality traverses its own noise axis
information-uniformly. By construction, this places both modalities at the same \emph{information}
quantile. The implicit ordering above is thus deliberately removed.

These two grids represent two competing hypotheses. Either \emph{clusters should lead}, or \emph{information should be released uniformly within each modality}. We compare these empirically in Section \ref{subsec:ablations}. By default, we adopt the coupled per-modality schedule.

\subsection{Training-time noise sampling.}
The distribution of $\bm{\sigma}$ in ~\eqref{equ:hierarchical_objective} admits two natural choices,
mirroring the multi- versus single-schedule distinction of
\citet{latentforcing}: sampling the $\sigma_j$ \emph{independently},
which trains the model on all relative noise configurations and permits
arbitrary inference trajectories; or sampling them \emph{coupled}
through a shared quantile, $v \sim \mathcal{U}[0, 1]$ and
$\sigma_j = exp(F_j^{-1}(v))$, which concentrates capacity on the trajectory
actually used at inference. We use the coupled variant by default.

For a budget of $N$ sampling steps, this yields a noise grid 
\begin{align}
\label{equ:coupled_grid}
    \bm{\sigma}_i
    = \big(\sigma_{i,1}, \dots, \sigma_{i,k}\big),
    \qquad
    \sigma_{i,j} = exp(F_j^{-1}\!\left(1 - \tfrac{i}{N}\right)),
    \qquad i = 0, \dots, N.
\end{align}

\subsection{Variations and links
}
\label{subsec:variations}

\paragraph{Beyond token clusters.}
The framework accommodates arbitrary hierarchies: deeper cluster trees
($k > 2$), or sequence-level semantic variables (e.g.,
sentence-level summaries), the latter requiring only a modality with
its own length and alignment. We evaluate the per-token, $k = 2$
instantiation and leave these to future work. \citet{wu2026rethinkingtokenpredictiontreestructured} especially consider discrete diffusion along hierarchical trees. They however find tree depth 2, that is $k=2$, to yield the best results.

\paragraph{Model Dependent Parameter tuning.}
Hierarchical Continuous Diffusion effectively designs parallel diffusion processes that are denoised jointly. This framework is applicable to any continuous DLM. As such, each specific model's sampling mechanisms and parameters can be tuned \emph{per-modality}, as we show for the entropy gated schedule and the stochastic churn in Section \ref{subsec:dual_noise}. This model specific tuning is expected to enhance the interplay between modalities. We explore varying per-modality stochastic samplers for the special case of H-CoBit, with substantial sample quality improvements (see Tables \ref{tab:main_table}, \ref{tab:dual_churn}).

\paragraph{Links to discrete diffusion.} Our method can be seen as a continuous diffusion analog to HDLM \citep{zhou2025next}, which unmasks \emph{high level} clusters before tokens. We draw from their clustering method, and also present a framework for $k$ arbitrary modalities evaluated in the special case $k=2$. In our method however, all modalities are still diffused \emph{jointly}, despite modality-specific schedulers. HDLM is based on masked diffusion, and hence fully predicts a token's cluster before predicting the actual token. 

\section{Experimental Details}
\label{app:exp_details}

We take CoBit's training and sampling recipe unchanged \citep{cobit}, and only add the cluster modality of Section~\ref{sec:method}. CoBit and H-CoBit share every hyperparameter except the cluster bits and $\lambda_2$. Table~\ref{tab:hparams} summarizes the per-dataset settings.

\paragraph{Architecture.} A 12-layer, 12-head Diffusion Transformer with hidden size 768, SwiGLU feed-forward of width 3072, RoPE, adaLN noise conditioning, dropout $0.1$, and CoBit's matched-filter residual output head \citep{cobit}. Each token's bit code is one patch, and H-CoBit appends the $\lceil\log_2 n\rceil=4$ cluster bits to it. The denoiser is conditioned on $\sigma^1$ only, which keeps the backbone's conditioning identical to the baseline.

\paragraph{Optimization.} AdamW, learning rate $3\cdot10^{-4}$, $2.5$k warm-up steps, gradient clipping at $1$, EMA $0.9999$, global batch size $512$, bf16. Initial training noise is log-normal ($P_{\mathrm{mean}}{=}-1.2$, $P_{\mathrm{std}}{=}1.2$) with EDM weighting and $\sigma_{\mathrm{data}}{=}1/2$ \citep{karras2022edm}, and self-conditioning with $p_{sc}{=}0.5$ \citep{chen2022analogbits}. On LM1B and OWT the learning rate follows a cosine decay over a 1M-step horizon. The 600k (LM1B H-CoBit) and 750k (OWT, both arms) checkpoints are therefore taken before the decay completes. On TinyGSM we follow CoBit's TinyGSM recipe: a constant learning rate, $\beta_2{=}0.999$ and no weight decay. We train on 4--8 A100-64GB GPUs. The batch size is global and there is no gradient accumulation.

\paragraph{Information-uniform schedule.} We use CoBit's online entropy-rate estimator. A 40k-step log-linear warmup is followed by refits every 2k steps from a 128-bin loss histogram, each blended in over 10k steps. H-CoBit fits one schedule per modality.

\paragraph{Sampling.} DDIM probability-flow ordinary differential equation with EDM churn self-conditioning \citep{cobit}, stopping at $\sigma_{\min}{=}0.08$. We sweep the churn $\gamma$ jointly with NFE $\in\{64,\dots,1024\}$. The per-modality sampler churns clusters in the late window $W_{\text{late}}=[0,0.5]$ only, with $\gamma_2$ as given in Table~\ref{tab:main_table}.

\paragraph{Evaluation.} GenPPL is GPT-2 Large perplexity over 1024 unconditional samples, which are decoded to text and retokenized. Entropy is the token unigram entropy per sample, averaged over samples, and the real-text references are 4.34 (LM1B) and 5.46 (OWT). MAUVE uses GPT-2 Large features on $1024$ samples against held-out text, and we report the mean over 6 $k$-means seeds \citep{franca2026hackinggenerativeperplexityunconditional}. For GSM8K, following \cite{deschenaux2026languagemodelinghypersphericalflows}, we train conditionally on the question, keeping the prompt noise-free and applying the loss to the answer only. We sample one answer for each of the 1319 problems and score the exact match of the final numeric answer.
\begin{table}[t]
\centering
\caption{\textbf{Main-run settings for H-CoBit.} All other hyperparameters follow \citet{cobit}.}
\label{tab:hparams}
\footnotesize
\setlength{\tabcolsep}{4pt}
\begin{tabular}{@{}lccc@{}}
\toprule
 & LM1B & OpenWebText & TinyGSM \\
\midrule
Tokenizer                      & BERT-base-uncased  & \texttt{gpt2id\_bpe16} & SmolLM-135M \\
Vocabulary $|\mathcal{V}|$     & 30,522             & 65,536                 & 49,153 \\
Sequence length $L$            & 128                & 1024                   & 512 \\
Bits/token (CoBit\,/\,H-CoBit) & $15\to19$          & $16\to20$              & $16\to20$ \\
Steps (CoBit\,/\,H-CoBit)      & 1M\,/\,600k        & 750k\,/\,750k          & 250k\,/\,250k \\
LR schedule                    & cosine (1M)        & cosine (1M)            & constant \\
$(\beta_1,\beta_2)$            & $(0.9,0.99)$       & $(0.9,0.99)$           & $(0.9,0.999)$ \\
Weight decay                   & $0.01$             & $0.01$                 & $0$ \\
Clusters $n$ / weight $\lambda_2$ & 16\,/\,0.27   & 16\,/\,$0.1$           & 16\,/\,$0.3$ \\
Cluster embeddings             & LangFlow-LM1B      & LangFlow-OWT (span mean) & S-FLM \\
\bottomrule
\end{tabular}
\end{table}

\paragraph{Extending to FLM.} We also instantiate our H-CDLM framework on FLM, yielding H-FLM. For that run, we use the exact same training and sampling hyperparameters as \citet{lee2026flm}. The cluster loss weight is set at $\lambda_2=1.0$. For efficiency reasons at sampling, we drop the cluster head and use the cluster probabilities implied by the token head, that is we compute cluster probabilities via the sum of the token probabilities of each cluster (see equation \ref{equ:token_implied_cluster}). We apply stochastic sampling to FLM, unlike its authors, by following the procedure described in Section \ref{sec:background}. For FLM, we set $s_{\text{noise}}=1.0$, and $\gamma=0.03$ which improves over the deterministic sampling of the original paper. For H-FLM, we set $\gamma=0.3$,  $s_{\text{noise}}=1.0$ and vary cluster churn and window (see Section \ref{subsec:h_flm}).

\section{Additional Ablations}
\label{app:add_exp}
\subsection{MAUVE evaluation of H-CoBit}

We evaluate MAUVE at matched entropy levels and report results in Table \ref{tab:MAUVE}. At dataset entropy and above, and at matched or lower NFE, H-CoBit beats the baseline by approximately 0.03-0.04 and 0.05-0.11 on LM1B and OWT respectively. Below dataset entropy, CoBit surpasses our model.

\begin{table*}[h!]\centering
\caption{\textbf{\textsc{mauve} at fixed entropy levels, best budget per level.} Entropy-matched MAUVE ($H_1$ = $+0.02$ nats off it, $H_{\mathrm{data}}$ = the dataset's own, $H_2$ = $-0.05$ nats off it); cell is the best \textsc{mauve} any swept NFE reaches at that level, winning NFE in parentheses. Target entropies (nats): LM1B 4.36/4.34/4.29; OWT 5.48/5.46/5.41.}

\begingroup\small
\setlength{\tabcolsep}{6pt}
\begin{tabular}{@{}lrrrrr@{}}
\toprule
 &  &  & \multicolumn{3}{c}{\textbf{MAUVE $\uparrow$ at fixed entropy (best NFE)}} \\
\cmidrule(lr){4-6}
\textbf{Model} & \textbf{steps} & \textbf{$\lambda_2$} & \textbf{$H_1$} & \textbf{$H_{\mathrm{data}}$} & \textbf{$H_2$} \\
\midrule
\multicolumn{6}{@{}l}{\itshape \underline{LM1B} \textnormal{(final checkpoints)}} \\
\textcolor{cmuted}{\itshape target $H$} &  &  & \textcolor{cmuted}{\itshape $4.36$ ($H_1$)} & \textcolor{cmuted}{\itshape $4.34$ ($H_{\mathrm{data}}$)} & \textcolor{cmuted}{\itshape $4.29$ ($H_2$)} \\
CoBit & 1M & -- & 0.870 (64) & 0.861 (128) & \textbf{0.697 (256)} \\
H-CoBit & 600k & $\frac{4}{15}$ & \textbf{0.900 (64)} & \textbf{0.903 (128)} & 0.648 (512) \\
\midrule
\multicolumn{6}{@{}l}{\itshape \underline{OpenWebText} \textnormal{(final checkpoints)}} \\
\textcolor{cmuted}{\itshape target $H$} &  &  & \textcolor{cmuted}{\itshape $5.48$ ($H_1$)} & \textcolor{cmuted}{\itshape $5.46$ ($H_{\mathrm{data}}$)} & \textcolor{cmuted}{\itshape $5.41$ ($H_2$)} \\
CoBit & 750k & -- & 0.329 (256) & 0.530 (128) & \textbf{0.613 (128)} \\
H-CoBit & 750k & 0.1 & \textbf{0.435 (128)} & \textbf{0.580 (128)} & 0.563 (128) \\
\bottomrule
\end{tabular}
\endgroup
  \label{tab:MAUVE}
\end{table*}

\subsection{Training and sampling speed comparison of H-CoBit}
\label{app:speed}

Table \ref{tab:speed} provides training and sampling speeds of CoBit-S and H-CoBit. As expected due to the larger channel width of H-CoBit, we observe a small sampling and training time overhead. Figure \ref{fig:teaser} and Table \ref{tab:main_table} show that we match CoBit-S' quality at lower NFE, which more than compensates for this speed gap. 

\begin{table*}[h]\centering
\caption{\textbf{Inference and training costs.} Both arms built and timed in one process on one A100-SXM-64GB; training is forward, backward and optimizer step with no dataloading, inference is deterministic sampling with self-conditioning at NFE=128. Per-call time is the inference unit because it is what multiplies by NFE to give generation time. Each dataset is read at the largest batch it was profiled at, where the GPU is saturated.}

\begingroup\small
\setlength{\tabcolsep}{4.5pt}
\begin{tabular}{@{}lcrrrrrr@{}}
\toprule
 &  & \multicolumn{3}{c}{\textbf{Training (ms/step)}} & \multicolumn{3}{c}{\textbf{Inference ($\mu$s/call)}} \\
\cmidrule(lr){3-5}\cmidrule(lr){6-8}
\textbf{Dataset} & \textbf{bits/tok} & \textbf{CoBit} & \textbf{H-CoBit} & \textbf{$\Delta$} & \textbf{CoBit} & \textbf{H-CoBit} & \textbf{$\Delta$} \\
\midrule
LM1B & 15$\to$19 & 123.5 & 125.8 & +1.9\% & 292 & 301 & +2.9\% \\
TinyGSM / GSM8K & 16$\to$20 & 231.5 & 237.3 & +2.5\% & 1174 & 1213 & +3.3\% \\
OpenWebText & 16$\to$20 & 244.3 & 250.0 & +2.3\% & 2401 & 2503 & +4.2\% \\
\quad $n{=}64$ & 16$\to$22 & 244.3 & 252.2 & +3.2\% & 2401 & 2520 & +5.0\% \\
\bottomrule
\end{tabular}
\endgroup

  \label{tab:speed}
\end{table*}

\subsection{A Co-Bit Specificity: Appended \emph{vs} Embedded Cluster Bits}
\label{subsubsec:embedded_bit_ablation}

CoBit \citep{cobit} represents tokens continuously as bit sequences. For this specific representation, it would hence be possible to embed the cluster indexing of tokens directly in their bit code, instead of adding additional bits to encode it.
Figure \ref{fig:embedded_v_appended} compares this \emph{embedded bits} strategy to our main \emph{appended} modality approach for LM1B at 100k training steps. The \emph{embedded} approach yields worse than baseline results for $n=16$, indicating that an independent cluster representation is necessary for our hierarchical architecture to yield improvements.

\begin{figure*}[h!]\centering
\begin{tikzpicture}
  \begin{axis}[
    paperaxis,
    width=12.44726127577141cm,
    height=6.5cm,
    xmin=4.08,
    xmax=4.385,
    ymin=25,
    ymax=165,
    xlabel={token unigram entropy $H$ (nats)},
    ylabel={Gen-PPL (GPT-2 large)},
    scaled x ticks=false,
    scaled y ticks=false,
    xtick={4.1,4.15,4.2,4.25,4.3,4.35},
    ytick={40,70,100,130,160},
  ]
    \addplot+[color=cdata, densely dashed, line width=0.8pt, mark=none, forget plot]
      table[row sep=\\] {
        4.34346 25\\
        4.34346 165\\
      };
    \addplot+[color=ccobit, solid, line width=1.1pt, mark=*, mark size=1.1pt, mark size=1.0pt, mark options={solid, fill=ccobit, draw=ccobit}, only marks, draw opacity=0.25, fill opacity=0.25, forget plot]
      table[row sep=\\] {
        4.12551 45.0085\\
        4.22398 56.4666\\
        4.2257 54.405\\
        4.23447 55.5772\\
        4.24344 58.1337\\
        4.24819 58.1528\\
        4.25756 60.912\\
        4.26349 65.036\\
        4.2789 68.9974\\
        4.29404 71.4594\\
        4.29471 71.4626\\
        4.29589 74.153\\
        4.29662 72.7658\\
        4.30005 73.6933\\
        4.30609 78.4606\\
        4.30766 80.5088\\
        4.31873 95.9965\\
        4.31955 92.9019\\
        4.32019 93.7627\\
        4.32097 95.4273\\
        4.32111 102.074\\
        4.32161 95.9159\\
        4.32247 96.4428\\
        4.32692 100.449\\
        4.32744 107.649\\
        4.3292 108.495\\
        4.331 109.893\\
        4.33216 139.67\\
        4.33259 149.109\\
        4.33548 142.015\\
        4.33562 143.285\\
        4.33638 141.548\\
      };
    \addplot+[color=ccobit, solid, line width=1.1pt, mark=*, mark size=1.1pt, line width=1.2pt, mark size=1.4pt, mark options={solid, fill=ccobit, draw=ccobit}, forget plot]
      table[row sep=\\] {
        4.09619 41.7042\\
        4.09954 42.5423\\
        4.10634 42.937\\
        4.12582 44.703\\
        4.14338 46.9577\\
        4.18637 49.1474\\
        4.21144 51.7083\\
        4.23924 52.5978\\
        4.25631 57.7451\\
        4.25779 59.4726\\
        4.26435 61.9033\\
        4.28001 66.2503\\
        4.28585 66.9048\\
        4.29007 70.111\\
        4.29688 70.6381\\
        4.30225 71.6057\\
        4.30884 76.3962\\
        4.31637 79.4646\\
        4.31867 84.9256\\
        4.32355 88.8478\\
        4.3262 93.921\\
        4.32708 98.3787\\
        4.32883 99.4723\\
        4.33259 100.094\\
        4.33449 107.679\\
        4.33714 139.262\\
        4.33796 143.435\\
      };
    \addplot+[color=cclust, dashed, line width=1.1pt, mark=*, mark size=1.1pt, mark size=1.0pt, mark options={solid, fill=cclust, draw=cclust}, only marks, draw opacity=0.25, fill opacity=0.25, forget plot]
      table[row sep=\\] {
        4.29708 66.1634\\
        4.30594 67.9791\\
        4.30806 66.8964\\
        4.34148 79.651\\
        4.34587 79.8108\\
        4.34611 83.3163\\
        4.3487 81.3574\\
        4.35202 104.826\\
        4.3525 146.664\\
        4.35369 108.449\\
        4.3537 113.607\\
        4.35761 143.717\\
        4.35846 107.297\\
        4.36056 143.766\\
        4.363 138.166\\
        4.36339 145.684\\
        4.36431 139.395\\
      };
    \addplot+[color=cclust, dashed, line width=1.1pt, mark=*, mark size=1.1pt, line width=1.2pt, mark size=1.4pt, mark options={solid, fill=cclust, draw=cclust}, forget plot]
      table[row sep=\\] {
        4.28696 61.8039\\
        4.31122 64.5132\\
        4.31926 70.7631\\
        4.35397 78.4143\\
        4.35552 84.8937\\
        4.35673 99.6486\\
        4.35764 101.851\\
        4.36359 101.998\\
        4.36436 102.845\\
        4.36795 138.841\\
      };
    \addplot+[color=cprefix, densely dash dot, line width=1.1pt, mark=*, mark size=1.1pt, mark size=1.0pt, mark options={solid, fill=cprefix, draw=cprefix}, only marks, draw opacity=0.25, fill opacity=0.25, forget plot]
      table[row sep=\\] {
        4.10889 42.0652\\
        4.11082 42.5463\\
        4.13393 45.0559\\
        4.23109 54.1224\\
        4.23223 56.5959\\
        4.24229 55.2143\\
        4.25271 58.6985\\
        4.25528 59.8721\\
        4.25531 60.8015\\
        4.25968 65.8916\\
        4.288 71.9112\\
        4.30066 74.4447\\
        4.30129 73.0195\\
        4.30424 72.8056\\
        4.30492 76.0725\\
        4.30944 77.2149\\
        4.3115 79.5261\\
        4.3168 81.7875\\
        4.33437 111.294\\
        4.33708 106.611\\
        4.33833 106.189\\
        4.34001 116.375\\
        4.34012 106.761\\
        4.34053 108.695\\
        4.34139 107.45\\
        4.3416 110.339\\
        4.34165 106.177\\
        4.34193 107.058\\
        4.34274 110.327\\
        4.3532 157.077\\
        4.35396 158.51\\
        4.35686 156.904\\
      };
    \addplot+[color=cprefix, densely dash dot, line width=1.1pt, mark=*, mark size=1.1pt, line width=1.2pt, mark size=1.4pt, mark options={solid, fill=cprefix, draw=cprefix}, forget plot]
      table[row sep=\\] {
        4.11123 41.9692\\
        4.11194 43.733\\
        4.1358 44.5332\\
        4.24521 53.2461\\
        4.24536 54.4928\\
        4.24771 55.1352\\
        4.24785 56.6212\\
        4.25225 57.6106\\
        4.26065 58.0876\\
        4.27384 64.9207\\
        4.30526 71.7045\\
        4.30885 75.4991\\
        4.31042 75.566\\
        4.3122 76.5993\\
        4.31331 80.0055\\
        4.32111 80.2207\\
        4.32331 86.3566\\
        4.32654 88.9711\\
        4.33685 105.021\\
        4.3397 105.358\\
        4.34385 105.665\\
        4.34421 107.718\\
        4.3443 111.926\\
        4.35223 153.732\\
        4.3527 155.511\\
        4.3583 156.378\\
        4.36035 156.54\\
      };
    \addplot+[color=black, densely dashed, line width=0.8pt, mark=star, mark size=1.1pt, only marks, mark size=3.0pt, mark options={line width=0.9pt}, forget plot]
      table[row sep=\\] {
        4.34346 53.1392\\
      };
    \node[cdata, font=\scriptsize, rotate=90, anchor=west, inner sep=1pt] at ([xshift=6pt, yshift=3pt]axis cs:4.3435,25) {data entropy};
    \node[anchor=north west, draw=black!25, line width=0.4pt, fill=white, fill opacity=0.9, text opacity=1, inner sep=2.5pt, font=\scriptsize] at (rel axis cs:0.025,0.975) {\begin{tabular}{@{}l@{~}l@{}}\tikz[baseline=-0.5ex]{\draw[ccobit, solid, line width=1.1pt] (0pt,0pt)--(13pt,0pt);} & no cluster bits (15 b/token) \\
\tikz[baseline=-0.5ex]{\draw[cclust, dashed, line width=1.1pt] (0pt,0pt)--(13pt,0pt);} & appended (19 b/token) \\
\tikz[baseline=-0.5ex]{\draw[cprefix, densely dash dot, line width=1.1pt] (0pt,0pt)--(13pt,0pt);} & prefix (15 b/token)\end{tabular}};
  \end{axis}
\end{tikzpicture}
\input{figures/embedded_layout_ablation_caption}
  \label{fig:embedded_v_appended}
\end{figure*}

\subsection{Token-cluster agreement for H-CoBit}
\label{subsec:token_cluster agreement}

We study how the information conveyed by the token head and cluster head differs throughout the sampling process in Figure \ref{fig:token_cluster_agreement}. The token-implied cluster corresponds to the most likely cluster after marginalizing token probabilities: 
\begin{align}
\bar{p}_\theta\big(x^{2,i} = c \mid x_\sigma\big)
  \;=\; \sum_{w \in V:\, \Gamma(w) = c} p_\theta\big(x^{1,i} = w \mid x_\sigma\big),
  \qquad c \in \{1, \dots, n\},
  \label{equ:token_implied_cluster}
\end{align}
where $\bar{p}_{\theta}$ is the cluster distribution implied by the token block alone. For H-CoBit, $p_\theta\big(x^{1,i} = w \mid x_\sigma\big)$ is computed via the product of bit probabilities. We compare it with the cluster block's own prediction $p_\theta(x^{2,i}=c \mid \mathbf{x}_\sigma)$.

We plot agreement with the final token, that is the percentage of end-of-sampling tokens that are already inferred at a given sampling step.

We observe that the cluster head adds information on the token-implied cluster; its agreement with the final generated sample is consistently higher throughout sampling (except for very early steps). This shows that the cluster modality carries additional information 
and is therefore
not purely redundant with the tokens.

\begin{figure*}[h!]\centering
\input{figures/token_cluster_agreement_hcobit_body}
\input{figures/token_cluster_agreement_hcobit_caption}
  \label{fig:token_cluster_agreement}
\end{figure*}

\subsection{Per-modality Sampling Evaluation of H-CoBit}
\label{app:multimodal_churn}

Table \ref{tab:dual_churn} provides a comprehensive evaluation of per-modality sampling for H-CoBit. As detailed in Section \ref{subsec:ablations}, we test adjusting global noise $s_{\text{noise}}$, the churn window $\mathcal{W}$ (early or late noise half), and the churn level $\gamma$ \emph{per-modality}.
Per-modality sampling adjustments allow multiple degrees of freedom. On our main benchmark, GenPPL at dataset entropy, they allow substantial gains compared to the baseline.

\begin{table*}[h!]\centering
\caption{\textbf{Every multimodal-churn arm on both datasets for H-CoBit.} $\Delta$Gen-PPL: the mean \% change against the control at matched entropy over the interval both arms' envelopes cover (negative is better; \emph{--} where they share none). Bold is best. Magnitude alone has no effect, it must be combined with late churning to yield improvements.}

\begingroup\small
\setlength{\tabcolsep}{6pt}
\begin{tabular}{@{}lrrrrrr@{}}
\toprule
\multicolumn{2}{c}{\textbf{Gen-PPL}} &  & \multicolumn{4}{c}{\textbf{$\Delta$ vs.\ control (\%) at NFE}} \\
\cmidrule(lr){1-2}\cmidrule(lr){4-7}
\textbf{Arm} & \textbf{@$H_{\mathrm{data}}$ (NFE)} & \textbf{$H$ reach} & \textbf{64} & \textbf{128} & \textbf{256} & \textbf{512} \\
\midrule
\multicolumn{7}{@{}l}{\bfseries LM1B: $n{=}16$, 600k steps, $\lambda_2{=}4/15$, $H_{\mathrm{data}}=4.34$} \\
control (shared churn) & 57.3 {\scriptsize(256)} & 4.356 & 0 & 0 & 0 & 0 \\
\multicolumn{7}{@{}l}{\itshape Cluster-churn magnitude $\gamma_2$} \\
$\gamma_2{=}0$ (frozen) & 57.8 {\scriptsize(512)} & 4.356 & -1.5 & +8.2 & +4.7 & \textbf{-2.0} \\
$\gamma_2{=}0.1$ & 59.7 {\scriptsize(256)} & 4.357 & -3.5 & +0.9 & -0.8 & -1.1 \\
$\gamma_2{=}0.2$ & 58.1 {\scriptsize(256)} & 4.350 & -2.9 & \textbf{-2.1} & -0.1 & +2.0 \\
\multicolumn{7}{@{}l}{\itshape Cluster-churn window (phase)} \\
early, $q{\in}[0.5,1]$ & 62.3 {\scriptsize(256)} & 4.357 & +5.2 & +3.3 & +4.4 & +7.1 \\
late, $q{\in}[0,0.5]$ & 51.6 {\scriptsize(512)} & 4.365 & -4.2 & +3.7 & -1.1 & -- \\
late, $\gamma_2{=}0.3$ & 51.5 {\scriptsize(512)} & 4.364 & \textbf{-4.4} & -0.1 & -2.5 & -- \\
late, $\gamma_2{=}0.4$ & \textbf{49.4 {\scriptsize(512)}} & 4.361 & -0.8 & -- & \textbf{-4.8} & -- \\
\multicolumn{7}{@{}l}{\itshape Compositions} \\
late $\gamma_2{=}0.3$ + $s_{\mathrm{noise}}{=}1.05$ & 56.3 {\scriptsize(1024)} & 4.407 & -- & -- & -- & -- \\
late $\gamma_2{=}0.4$ + $s_{\mathrm{noise}}{=}1.05$ & 57.9 {\scriptsize(1024)} & \textbf{4.408} & -- & -- & -- & -- \\
\midrule
\multicolumn{7}{@{}l}{\bfseries OpenWebText: $n{=}16$, 750k steps, $\lambda_2{=}0.1$, $H_{\mathrm{data}}=5.46$} \\
control (shared churn) & 67.8 {\scriptsize(128)} & 5.645 & 0 & 0 & 0 & 0 \\
\multicolumn{7}{@{}l}{\itshape Cluster-churn magnitude $\gamma_2$} \\
$\gamma_2{=}0$ (frozen) & 72.6 {\scriptsize(128)} & 5.674 & +8.6 & -0.3 & -- & -- \\
\multicolumn{7}{@{}l}{\itshape Cluster-churn window (phase)} \\
late, $\gamma_2{=}0.3$ & 69.2 {\scriptsize(128)} & 5.649 & +3.6 & -4.1 & -- & -- \\
\multicolumn{7}{@{}l}{\itshape Global noise scale $s_{\mathrm{noise}}$ (both modalities)} \\
$s_{\mathrm{noise}}{=}1.02$ & 71.9 {\scriptsize(128)} & 5.659 & -5.5 & -3.3 & +6.2 & +1.9 \\
$s_{\mathrm{noise}}{=}1.05$ & 80.0 {\scriptsize(128)} & 5.709 & -10.7 & -1.1 & +14.5 & +6.8 \\
$s_{\mathrm{noise}}{=}1.10$ & -- & \textbf{5.755} & \textbf{-13.6} & +6.3 & -- & -- \\
\multicolumn{7}{@{}l}{\itshape Compositions} \\
late $\gamma_2{=}0.3$ + $s_{\mathrm{noise}}{=}1.05$ & \textbf{50.4 {\scriptsize(256)}} & 5.711 & -8.1 & \textbf{-15.4} & -- & -- \\
\bottomrule
\end{tabular}
\endgroup

  \label{tab:dual_churn}
\end{table*}

\section{Generation Examples}
\label{app:gen_examples}

In this section, we show examples of unconditional text generation of our best performing models on LM1B and OWT. We report text sequences corresponding to the hyperparameters that yield the best GenPPL at dataset entropy. 

\begin{figure}[h]
    \centering
    \input{figures/text_samples_lm1b_body}
    \input{figures/text_samples_lm1b_caption}
\end{figure}

\begin{figure}[!h]
    \centering
    \input{figures/text_samples_owt_body}
    \input{figures/text_samples_owt_caption}
\end{figure}

\section{Limitations and Extensions}
\label{subsec:limitations_extensions}

\subsection{Limitations}
We study the simplest instantiation of H-CDLM, which uses a single auxiliary modality obtained from a fixed clustering of pretrained token embeddings. First, we consider only one level of granularity. Richer hierarchies spanning multiple levels, from fine-grained clusters to broad semantic categories, are a natural extension that our framework supports but that we leave unexplored. Second, our cluster assignments are fixed before training. Learning them jointly with the model could yield groupings better suited to denoising, at the cost of additional training complexity. Finally, our experiments are limited to relatively small models, reflecting the computational resources available in an academic setting. As for diffusion language models more broadly, whether these gains persist at larger scales remains to be established.

\subsection{Extensions}

\paragraph{Cluster guidance.}
The high-level channel provides a natural handle for guidance: as in
multimodal image diffusion \citep{latentforcing}, classifier-free guidance can be applied on the cluster modality to strengthen semantic coherence. We leave this to future work.

\end{document}